%% file: main.tex
\documentclass{article}

\usepackage{microtype}
\usepackage{graphicx}
\usepackage{subfigure}
\usepackage{booktabs} 
\usepackage{multirow}
\usepackage{array} 
\usepackage[table]{xcolor}
\usepackage{xcolor}
\usepackage{hyperref}

\usepackage[accepted]{icml2024}

\usepackage{amsmath}
\usepackage{amssymb}
\usepackage{mathtools}
\usepackage{amsthm}
\usepackage{xstring}
\usepackage{enumitem}

\newcommand{\Method}{Diversified Batch Selection}
\newcommand{\method}{DivBS}
\newcommand{\Methodspace}{Diversified Batch Selection }
\newcommand{\Methodspaceunderline}{\underline{Div}ersified \underline{B}atch \underline{S}election }
\newcommand{\methodspace}{DivBS }
\newcommand{\titlename}{Diversified Batch Selection for Training Acceleration}
\newcommand{\cleantitlename}{\StrSubstitute{\titlename}{\\}{}[\temp]\temp}

\definecolor{mediumpersianblue}{rgb}{0.0, 0.4, 0.65}
\definecolor{citecolor}{RGB}{0, 50, 110}
\definecolor{linkcolor}{RGB}{150, 50, 50}
\definecolor{darkred}{RGB}{139, 0, 0}
\usepackage[textsize=tiny]{todonotes}

\usepackage[capitalize,noabbrev]{cleveref}

\theoremstyle{plain}
\newtheorem{theorem}{Theorem}[section]
\newtheorem{proposition}[theorem]{Proposition}
\newtheorem{lemma}[theorem]{Lemma}

\theoremstyle{definition}
\newtheorem{definition}[theorem]{Definition}

\theoremstyle{remark}

\graphicspath{{fig/}}

\def\gD{{\mathcal{D}}}
\def\gE{{\mathcal{E}}}

\newcommand{\IE}{\emph{i.e.}}
\newcommand{\EG}{\emph{e.g.}}
\DeclareMathOperator{\st}{s.t.}
\icmltitlerunning{\cleantitlename}
\urldef{\urlinfobatch}\url{https://github.com/NUS-HPC-AI-Lab/InfoBatch/issues/16#issuecomment-1903467666}

\begin{document}
\twocolumn[
\icmltitle{\titlename}



\icmlsetsymbol{intern}{\dag}

\begin{icmlauthorlist}
\icmlauthor{Feng Hong}{sjtu,intern}
\icmlauthor{Yueming Lyu}{cfar,ihpc}
\icmlauthor{Jiangchao Yao}{sjtu,ailab}
\icmlauthor{Ya Zhang}{sjtu,ailab}
\icmlauthor{Ivor W. Tsang}{cfar,ihpc,ntu}
\icmlauthor{Yanfeng Wang}{sjtu,ailab}
\end{icmlauthorlist}

\icmlaffiliation{sjtu}{Cooperative Medianet Innovation Center, Shanghai Jiao Tong University, Shanghai, China}
\icmlaffiliation{cfar}{CFAR, Agency for Science, Technology and Research (A*STAR), Singapore}
\icmlaffiliation{ihpc}{IHPC, Agency for Science, Technology and Research (A*STAR), Singapore}
\icmlaffiliation{ailab}{Shanghai AI Laboratory, Shanghai, China}
\icmlaffiliation{ntu}{College of Computing and Data Science, NTU, Singapore}

\icmlcorrespondingauthor{Jiangchao Yao}{Sunarker@sjtu.edu.cn}
\icmlcorrespondingauthor{Yanfeng Wang}{wangyanfeng@sjtu.edu.cn}

\vskip 0.3in
]



\printAffiliationsAndNotice{\textsuperscript{\dag}Work down during internship at CFAR, A*STAR}  
\input{Tex/Abs.tex}
\input{Tex/Intro.tex}

\input{Tex/Background.tex}

\input{Tex/Method.tex}

\input{Tex/Experiments.tex}

\input{Tex/RelatedWork.tex}

\input{Tex/Conclusion}

\bibliography{main}
\bibliographystyle{icml2024}

\input{Tex/Appendix}

\end{document}

%% file: Tex/Abs.tex
\begin{abstract}
    The remarkable success of modern machine learning models on large datasets often demands extensive training time and resource consumption.  
    To save cost, a prevalent research line, known as online batch selection, explores selecting informative subsets during the training process.
    Although recent efforts achieve advancements by measuring the impact of each sample on generalization, their reliance on additional reference models inherently limits their practical applications, when there are no such ideal models available. On the other hand, the vanilla reference-model-free methods involve independently scoring and selecting data in a sample-wise manner,
    which sacrifices the diversity and induces the redundancy.
     To tackle this dilemma, we propose \Methodspaceunderline (\method), which is reference-model-free and can efficiently select diverse and representative samples. 
Specifically, we define a novel selection objective that measures the group-wise orthogonalized representativeness to combat the redundancy issue of previous sample-wise criteria, and provide a principled selection-efficient realization.
Extensive experiments across various tasks demonstrate the significant superiority of \methodspace in the performance-speedup trade-off. The \href{https://github.com/Feng-Hong/DivBS}{code} is publicly available.
\end{abstract}

%% file: Tex/Intro.tex
\section{Introduction}
Deep learning, propelled by vast amounts of web-scraped data, has led to significant advancements in models such as GPT-4~\citep{DBLP:journals/corr/abs-2303-08774}, CLIP~\citep{DBLP:conf/icml/RadfordKHRGASAM21}, SAM~\citep{Kirillov_2023_ICCV}, and Stable Diffusion~\citep{Rombach_2022_CVPR}. However, the time-intensive training process, lasting for weeks or even months, poses challenges with extended development cycles and increased resource consumption. Additionally, with a growing focus on data quality in deep learning systems, given the prevalence of low-quality, redundant, and biased data in real-world scenarios~\citep{DBLP:journals/corr/abs-2302-03169,deng2023towards}, there is an increasing need to select valuable training data for accelerating model training while maintaining the performance.

Recent studies~\citep{pmlr-v162-mindermann22a,deng2023towards} have achieved notable acceleration and convergence results by employing the online batch selection~\citep{Loshchilov2015OnlineBS} paradigm, which involves selecting samples that are most conducive to model convergence at the current training stage.
However, these reference-model-based methods rely on extra reference models, either trained from a considerable amount of holdout data~\citep{pmlr-v162-mindermann22a} or a pre-trained zero-shot predictor~\citep{deng2023towards}. Obtaining such a reference model can be costly or challenging in certain scenarios, especially for large-scale pre-training tasks. On the other hand, reference-model-free online batch selection methods~\citep{Jiang2019AcceleratingDL,pmlr-v80-katharopoulos18a,Loshchilov2015OnlineBS} prioritize challenging samples based on high loss or large gradient norm. Despite their practicality and efficiency, they often fall short in performance even compared to uniform selection~\citep{pmlr-v162-mindermann22a,deng2023towards}. 

\begin{figure*}[t]
    \vskip 0.2in
\subfigure[Full Data]{
    \begin{minipage}{0.233\textwidth}
    \centering     
    \includegraphics[width=\textwidth]{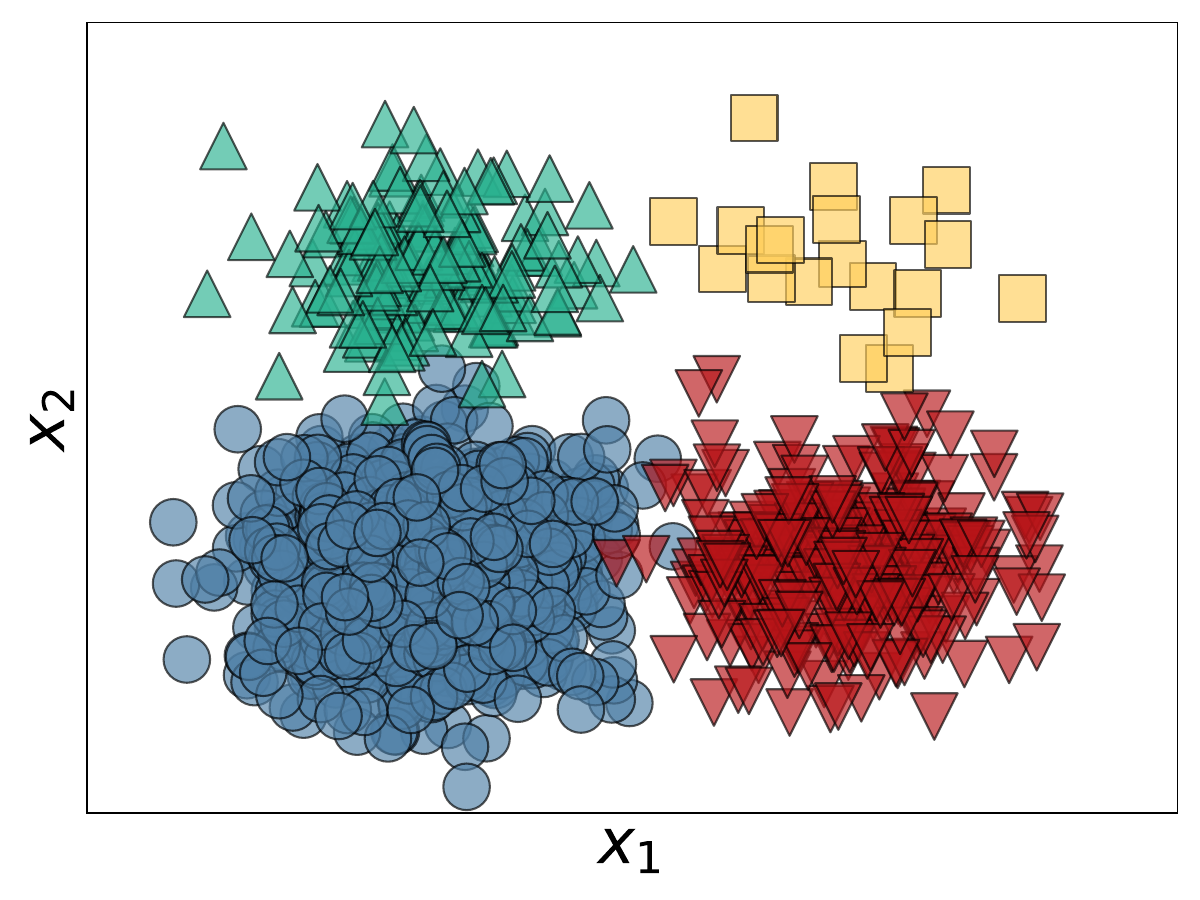}
    \end{minipage}
    }
\subfigure[Uniform]{
    \begin{minipage}{0.233\textwidth}
    \label{fig:toy uniform}
    \centering     
    \includegraphics[width=\textwidth]{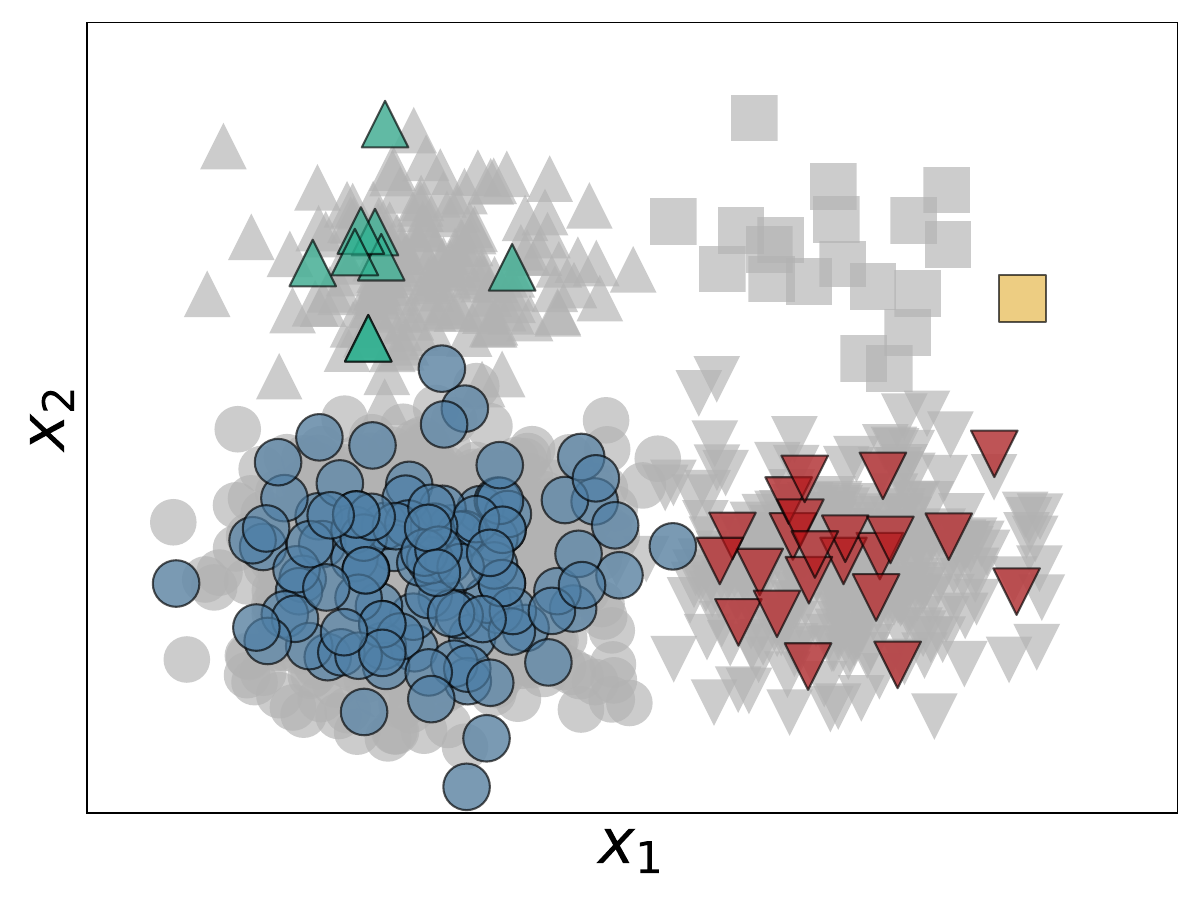}
    \end{minipage}
    }
\subfigure[Train Loss]{
    \begin{minipage}{0.233\textwidth}
    \centering     
    \label{fig:toy train loss}
    \includegraphics[width=\textwidth]{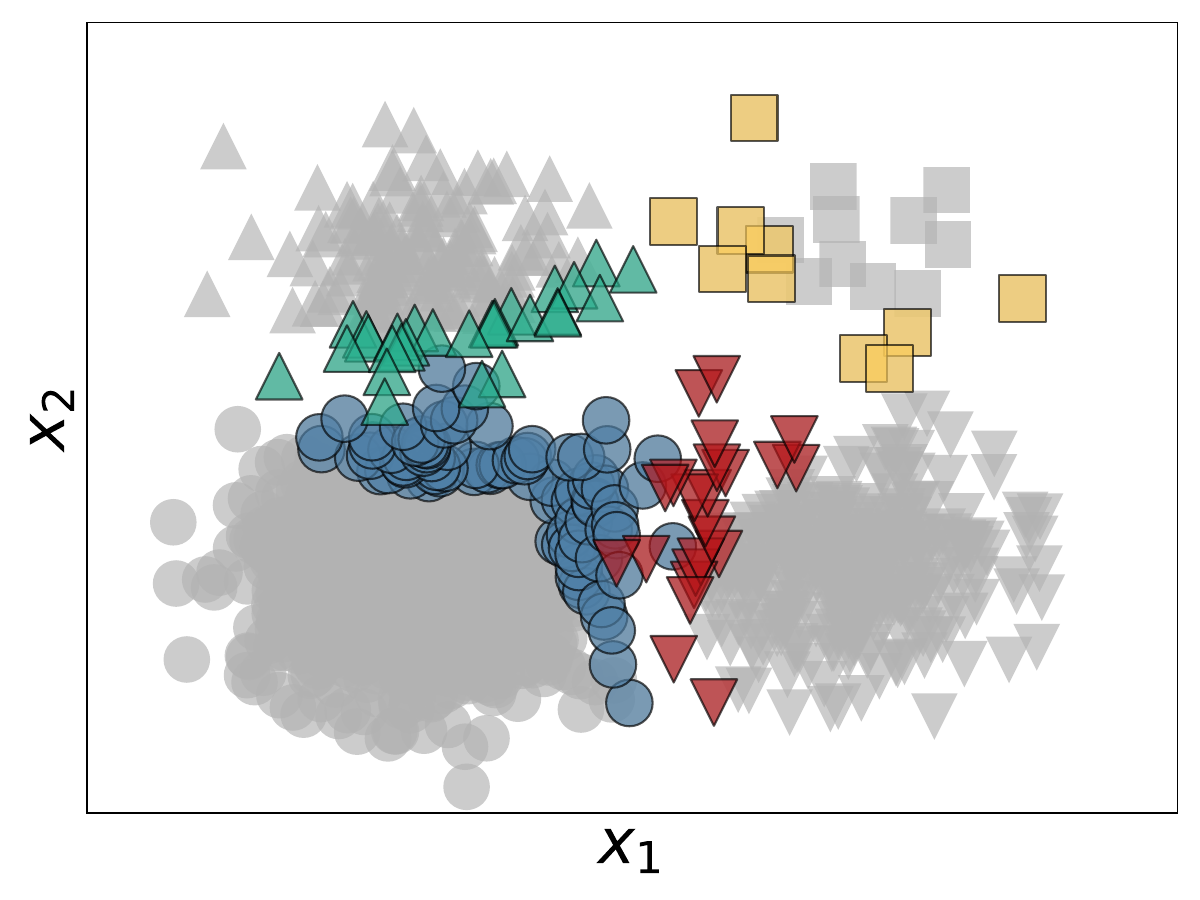}
    \end{minipage}
    }
\subfigure[DivBS]{
    \begin{minipage}{0.233\textwidth}
    \label{fig:toy DivBS}
    \centering     
    \includegraphics[width=\textwidth]{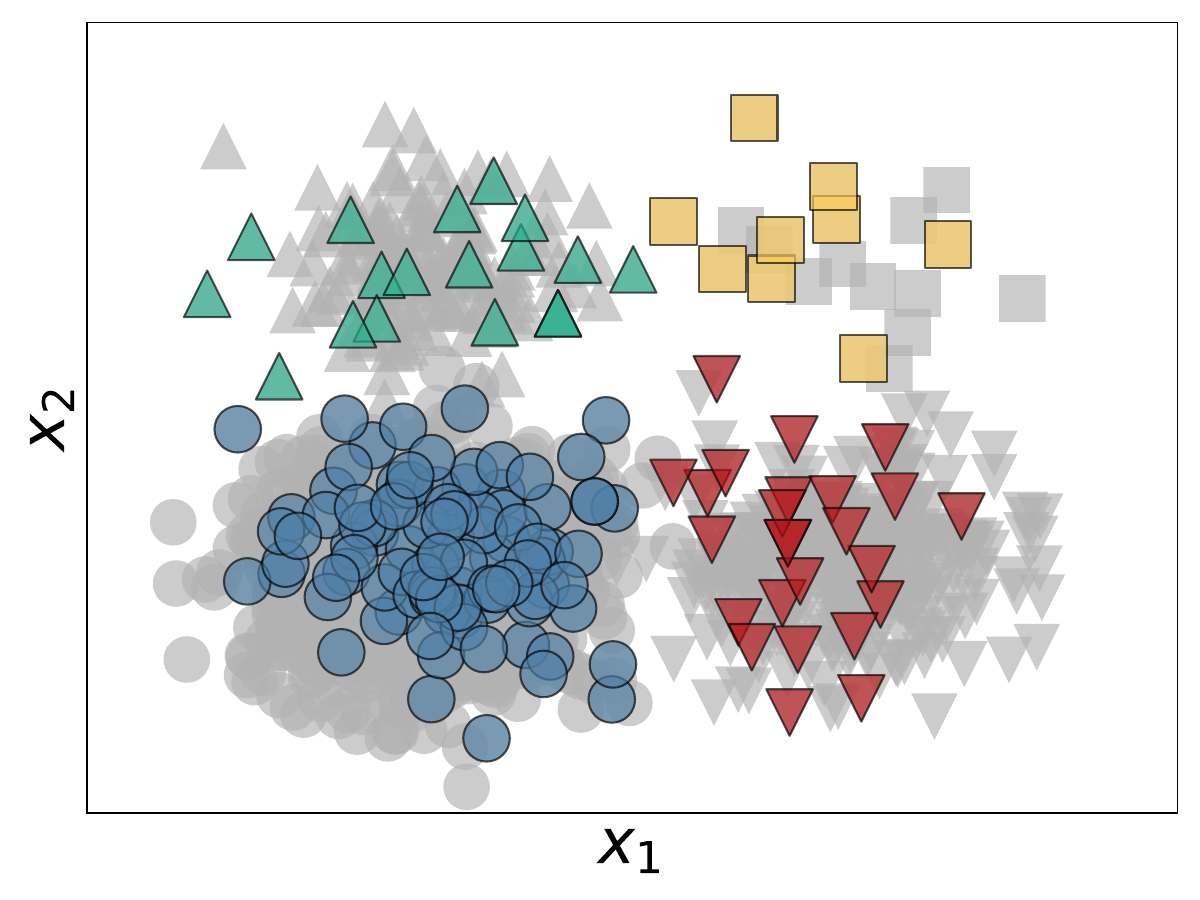}
    \end{minipage}
    }
\caption{Visualization of a toy motivating example, which is a 2D imbalanced four-class classification problem. Subfigure (a)  represents all the training data. Subfigures (b), (c), and (d) depict the subsets selected by the Uniform, Train Loss, and DivBS methods, with 10\% budget. For more details, please refer to \cref{appendix:toy}.}
\label{fig: toy}
\end{figure*}

In this paper, we focus on selecting a limited budget of crucial samples in a reference-model-free batch selection manner for training acceleration with negligible performance drop. We contend that existing reference-model-free selection methods adopt sample-wise strategies. Specifically, they independently apply predefined scoring criteria to all samples and conduct the score-based selection. Such \emph{sample-wise selection} methods overlook the correlations and redundancies among samples and may lead to poor diversity, degrading the performance.
When a sample is selected based on a high score, similar (or even identical) samples also receive similar scores. However, these samples contribute negligible new information to model training. As shown in \cref{fig:toy train loss}, samples selected based on train loss exhibit significant overlap and poor coverage of the original space. \citet{DBLP:conf/iclr/ZhengLL023,deepcore} have also discussed the negatives of such methods on subset diversity and performance in the context of the one-shot coreset selection, particularly with small budgets.

To tackle the diversity challenge, we propose a novel  reference-model-free batch selection method, \Methodspaceunderline (\method). Our core concept encompasses two aspects: (1) data selection should consider the selected subset as a whole, rather than independently selecting data on a sample-wise basis; (2) when assessing the overall representativeness of a subset, the inter-sample redundancy should be eliminated. Motivated by this, we introduce a new objective for batch selection, aiming to maximize the overall orthogonalized representativeness of the subset after removing the inter-sample redundancy
(\cref{eq:representativeness-subset,eq:objective}). 
Through a principled simplification of the optimization problem (\cref{prop:basis}), we propose a greedy algorithm (\cref{alg:greedy}) that can theoretically achieve an approximate ratio of $1-e^{-1}$ \emph{w.r.t.} the optimal objective value (\cref{prop:r guarantee}). We further streamline the selection process (\cref{alg:DivBS}), empirically achieving substantially reduced time consumption with comparable performance. We summarize the contributions as follows:
\begin{itemize}[itemsep=2pt,parsep=2pt,topsep=2pt,partopsep=2pt,leftmargin=15pt]
\item We explore prioritizing samples that enhance model convergence without extra reference models and highlight the diversity challenge faced by current sample-wise online batch selection 
methods.
\item We propose to maximize the overall orthogonalized representativeness of the subset rather than independently sample-wise selection. Building on a greedy algorithm with a $1-e^{-1}$ approximate ratio, we present a more efficient reference-model-free batch selection method, \Methodspace(\method).
\item We conduct extensive experiments, covering image classification, imbalanced learning, semantic segmentation, cross-modal retrieval, and language model fine-tuning. The results consistently demonstrate the superiority of \methodspace in accelerating training while maintaining the performance, \EG, with 70\% fewer iterations,  classification accuracy drops by under 0.5\% on average, segmentation mIoU decreases by under 1\%, and cross-modal retrieval performance improves.
\end{itemize}

%% file: Tex/Background.tex
\section{Background: Online Batch Selection}
We consider a task of 
learning a deep model $f_\theta$ with parameters $\theta$ on training data $\gD$ with stochastic gradient descent (SGD). 
At each training step, we can access a data batch $B = \{d_i\}_{i=1}^{N_B}$ with $N_B$ data points from $\gD$. In the online batch selection scenario, we need to conduct a smaller batch $S \subset B$ with a fixed budget of sample number $N_S<N_B$ to update the model. The next large batch is then pre-sampled from $\gD$ without replacement of previously sampled points in a same epoch.

Let \( U = g(B,\theta) = \{g(d_i,\theta)\}_{i=1}^{N_B} \) denote the features used for selection from \( B \), where $g$ denotes the mapping function from data points to the selection features given current model $f_{\theta}$. Existing methods simplify the problem of selecting a subset from $B$ into a sample ranking problem. Employing different scoring criteria $s(u), u\in U$, they select the top-$N_S$ samples from $B$ to conduct $S$. \citet{Loshchilov2015OnlineBS,Jiang2019AcceleratingDL} opt for high loss, where $U$ contains the outputs and labels, and $s(\cdot)$ is the loss function; \citet{pmlr-v80-katharopoulos18a} select samples with large gradient norm, where $U$ is the sample-wise gradients, and $s(\cdot)$ is the norm function; \citet{pmlr-v162-mindermann22a,deng2023towards} leverage a reference model to compute the score, where $U$ contains the training model outputs, the reference model outputs, and class labels, and $s(\cdot)$ is an approximate version of the generalization loss.
These methods select data in a sample-wise manner, without considering the interactions and redundancy among samples when they collectively update the model within a batch.

%% file: Tex/Method.tex
\section{Method: \Method}
\subsection{Motivation}
In \cref{fig: toy}, we present a toy example involving an imbalanced four-class classification task and showcase subsets selected by different batch selection methods. We can observe that the uniform sampling method (\cref{fig:toy uniform}) does not always achieve effective coverage of the original sample space, especially for the low-density (yellow and green) regions.
The method of selecting challenging samples (Train Loss) (\cref{fig:toy train loss}) results in high redundancy, with many points nearly overlapping. Additionally, the data distribution deviates significantly from the overall distribution, leading to inferior model convergence. We posit that this stems from current methods independently scoring and selecting data in a sample-wise manner. We propose that the overall evaluation of the selected samples should be conducted instead of sample-wise scoring, and the impact of inter-sample redundancy should be removed.
It is evident that the subset selected by our proposed \methodspace (\cref{fig:toy DivBS}) effectively covers the original sample space and significantly enhances the diversity of the selected samples.

\begin{algorithm}[t]
   \caption{The greedy algorithm.}
   \label{alg:greedy}
\begin{algorithmic}[1]
   \STATE {\bfseries Input:} Batch $B$, current model $f_{\theta}$, budget number $N_S$
   \STATE {\bfseries Output:} Selected mini batch $S$
   \STATE $S \leftarrow \emptyset$,  $E \leftarrow \emptyset$, $\mathrm{Sum} \leftarrow \sum_{u\in g(B,\theta)}u$
   \REPEAT
   \STATE $E_{\text{Cand}} \leftarrow \{\frac{g(d,\theta) - \sum_{e \in E} (e \cdot g(d,\theta)) e}{\|g(d,\theta) - \sum_{e \in E} (e \cdot g(d,\theta)) e\|} \text{ for } d \in B\}$ \\\textcolor{gray}{\texttt{\textit{// Candidate orthonormal basis}}}
   \STATE $\text{idx} \leftarrow \arg\max_{i} |e_i\cdot \mathrm{Sum}|, e_i \in E_{\text{Cand}}$ 
   \STATE $S \leftarrow S \cup \{d_{\text{idx}}\}$
   \STATE $E \leftarrow E \cup \{e_{\text{idx}}\}$
   \STATE $B \leftarrow B \setminus \{d_{\text{idx}}\}$
   \UNTIL{$|S| = N_S$}
\end{algorithmic}
\end{algorithm}

\subsection{Objective}
Recall that our core ideas for addressing the diversity challenge of the reference-model-free batch selection are: (1) we should consider the representativeness of the selected subset as a whole, rather than evaluating data in a sample-wise manner; (2) we should eliminate the impact of inter-sample redundancy on the representativeness of the subset. Motivated by these, we define a new measure, which characterizes the group-wise orthogonalized representativeness of subset \(S\) with respect to \(B\) as:
\begin{equation}
\label{eq:representativeness-subset}
	r(S,B,\theta)=\max_{E\in \gE(g(S,\theta))} \sum_{e\in E}\sum_{u\in g(B,\theta)}e\cdot u
\end{equation}
Here $\gE(g(S,\theta)) := \{E = \{e_1,\ldots e_{|E|}\} | \forall e_i, e_j \in E, e_i\cdot e_j = \delta_{ij}, \mathrm{span}(E) = \mathrm{span}(g(S,\theta))\}$ denotes the set of all potential orthonormal bases for the subspace spanned by $g(S,\theta)$, where \(\delta_{ij}\) is the Kronecker delta, taking the value $1$ when \(i = j\) and $0$ otherwise, and $\mathrm{span}(\cdot)$ denotes the subspace spanned by all elements in a set. 

Generally, the design intuition of \cref{eq:representativeness-subset} involves removing inter-sample redundancy in subset $S$ through orthogonalization. When calculating the contribution of an element in the subset $S$ to $r$, we subtract the redundant components already presented by other elements in $S$ and only consider its unique part distinct from others. This prevents redundant components between elements from contributing duplicate values to the orthogonalized representativeness $r$. In contrast, in sample-wise selection, redundant elements contribute duplicate scores because they are scored individually and then directly added up. Therefore, based on \cref{eq:representativeness-subset}, we propose a new objective for online batch selection, aiming to choose a subset $S\subset B$ with the maximum orthogonalized representativeness $r(S,B,\theta)$, under a specified budget constraint $|S|\leq N_S$.
\begin{equation}
	\label{eq:objective}
	\arg\max_{S\subset B} r(S,B,\theta),\quad \st |S|\leq N_S.
\end{equation}

\begin{algorithm}[t]
   \caption{\method.}
   \label{alg:DivBS}
\begin{algorithmic}[1]
   \STATE {\bfseries Input:} Batch $B$, current model $f_{\theta}$, budget number $N_S$
   \STATE {\bfseries Output:} Selected mini batch $S$
   \STATE $S \leftarrow \emptyset$,  $E \leftarrow \emptyset$, $\mathrm{Sum} \leftarrow \sum_{u\in g(B,\theta)}u$
   \REPEAT
   {\STATE $d \leftarrow \arg\max_{d \in B} |g(d,\theta) \cdot \mathrm{Sum}|$ }
   \STATE $e \leftarrow \frac{g(d,\theta) - \sum_{e \in E} (e \cdot g(d,\theta)) e}{\|g(d,\theta) - \sum_{e \in E} (e \cdot g(d,\theta)) e\|}$
   \STATE $S \leftarrow S \cup \{d\}$
   \STATE $E \leftarrow E \cup \{e\}$
   \STATE $\mathrm{Sum} \leftarrow \mathrm{Sum} - (e\cdot\mathrm{Sum})e$
   \textcolor{gray}{\texttt{\textit{// Subtracting the orthogonal components of the already selected samples from $\mathrm{Sum}$.}}}
   \STATE $B \leftarrow B \setminus \{d\}$
   \UNTIL{$|S| = N_S$ or $\mathrm{Sum}=0$}
\end{algorithmic}
\end{algorithm}

\subsection{Optimization}

Directly optimizing \cref{eq:objective} encounters two immediate challenges: (1) directly traversing an infinite space of orthonormal bases is clearly impractical; (2) the candidate subset space with size \(O({N_B}^{N_S})\) is also challenging to efficiently solve. In the following, we present our development to address these two challenges.

\textbf{Simplify the optimization of the orthonormal basis $E \in \gE(g(S,\theta))$.} To tackle the first challenge, we introduce \cref{prop:basis}, which offers a simple computational form for $r(S,B,\theta)$ \emph{w.r.t.} \emph{any} orthonormal basis for $g(S,\theta)$. 

\begin{proposition}\label{prop:basis}
$\forall E \in \gE(g(S,\theta))$, we have 
$$\sqrt{|E|\sum_{e\in E} (e\cdot \sum_{u\in g(B,\theta)}u)^2 } = r(S,B,\theta).
$$
\end{proposition}
The computational form for $r(S,B,\theta)$ in \cref{prop:basis} enables us to freely choose the orthonormal basis without affecting the optimization of the overall objective. Therefore, we no longer need to consider the optimization of $E \in \gE(g(S,\theta))$ and can directly optimize 
$S$ to maximize $r(S,B,\theta)$. Please refer to \cref{appendix:Proof of prop:basis} for the proof.

\begin{table}[t]
\centering
\caption{Final accuracy and wall-clock time of one epoch training of different methods on CIFAR-10 with 10\% budget.}
 \vskip 0.05in
\label{tab:DivBS-greedy}
\resizebox{0.99\linewidth}{!}{
\begin{tabular}{c|c|ccc}
\toprule
 Method    & Full    & Uniform & \cref{alg:greedy} & \cref{alg:DivBS}   \\
 \midrule\midrule
Acc $\uparrow$  & 95.50\% & 92.06\% & 94.57\%     & 94.65\% \\
Time $\downarrow$ & 44.7s   & 13.2s   & 19.8s       & 13.9s  \\\bottomrule
\end{tabular}}
\end{table}

\textbf{Greedy algorithm and its approximation guarantee.} Moreover, we seek to streamline the optimization for subset $S\subset B$. A straightforward approach involves sequentially and greedily selecting data that maximizes $r(S,B,\theta)$ according to the computational form in \cref{prop:basis}, as depicted in \cref{alg:greedy}. 

To analyze the theoretical guarantee of \cref{alg:greedy} for the optimization problem in \cref{eq:objective}, we introduce an auxiliary function $r'(S,B,\theta) = \sqrt{\sum_{e\in E} (e\cdot \sum_{u\in g(B,\theta)}u)^2 }$, $\forall E \in \mathcal{E}(g(S,\theta))$, which removes the $|E|$ term from $r(S,B,\theta)$ and also remains constant as $E \in \mathcal{E}(g(S,\theta))$ changes. In the following, we show \cref{prop:r'}, which provides some favorable properties of $r'(S,B,\theta)$.

\begin{proposition}
\label{prop:r'}
Define $r'$ as $r'(S,B,\theta) = \sqrt{\sum_{e\in E} (e\cdot \sum_{u\in g(B,\theta)}u)^2 }$, $\forall E \in \gE(g(S,\theta))$.  $r'(S,B,\theta)$ is submodular~\citep{DBLP:journals/ftml/Bach13}, normalized, and monotone (referring to \cref{definition:submodular} for detailed definition).
\end{proposition}

As per \cref{prop:r'}, \(r'(S,B,\theta)\) is submodular, normalized, and monotone for \(S\subset B\). Considering that \cref{alg:greedy} is also a greedy algorithm for maximizing \(r'(S,B,\theta)\), we can establish that \cref{alg:greedy} has a $1-e^{-1}$ approximation ratio \emph{w.r.t.} the optimal objective
value for maximizing \(r'(S,B,\theta)\)~\citep{DBLP:journals/mp/NemhauserWF78} (also referring to \cref{lemma:r' ratio} in Appendix). 

Furthermore, by establishing the connection between the optimal solutions for \(r(S,B,\theta)\) and \(r'(S,B,\theta)\), along with the solutions output by \cref{alg:greedy}, and their corresponding objective values, we present \cref{prop:r guarantee}. 

\begin{proposition}
\label{prop:r guarantee}
\cref{alg:greedy} returns a $1-e^{-1}$ approximation for $\arg\max_{S\subset B} r(S,B,\theta), \st |S|\leq N_S$. That is, denote $S^*$ as the optimum subset for \cref{eq:objective}, and $S'$ as the output of \cref{alg:greedy}, we have $$r(S',B,\theta) \geq (1-e^{-1}) r(S^*,B,\theta).$$
\end{proposition}

\cref{prop:r guarantee} provides the theoretical guarantee for \cref{alg:greedy} regarding the optimization problem defined in \cref{eq:objective} with an approximation ratio of $1-e^{-1}$. This enables us to effectively transform a subset optimization problem into a sequential selection problem. Please refer to \cref{appendix:Proof of prop:r guarantee} for the detailed proof.

\subsection{Realization}
\textbf{More Efficient Selection Process.} In \cref{alg:greedy}, the operation in line 5 involves subtracting the corresponding components of $E$ (an orthonormal basis of already selected samples) in all elements of $U = g(B,\theta)$, and then normalizing them. We further simplify this step in \cref{alg:DivBS} by only subtracting the orthogonal components on the $\mathrm{Sum}$ term, as shown in line 9. Please refer to \cref{appendix:plain-text for algs} for a detailed plain-text description of \cref{alg:greedy,alg:DivBS}. Here, we make an approximation by neglecting the normalization operation in \cref{alg:greedy}'s line 5; refer to \cref{appendix:Approximation from alg:greedy to alg:DivBS} for details. Basically, we simplify the process of orthogonalization of $|B|$ elements (line 5 of \cref{alg:greedy}) w.r.t. $E$ to the orthogonalization of only $1$ element and a subtraction operation (line 6 and 9 of \cref{alg:DivBS}), considerably reducing the cost of selection.
Empirically, we observed that \cref{alg:DivBS} substantially reduces the selection time compared to \cref{alg:greedy} while achieving comparable performance, as shown in \cref{tab:DivBS-greedy}.

\begin{table*}[t]
\centering
\caption{Final accuracies ($\uparrow$) of \methodspace and  various baseline methods on CIFAR-10, CIFAR-100 and Tiny ImageNet with different budget ratio 10\%, 20\%, 30\%. Bold indicates the best results. Experiments show that \methodspace consistently outperforms all baselines.}
\label{tab: classification-main}
 \vskip 0.02in
\resizebox{\textwidth}{!}{
\begin{tabular}{@{}!{\vrule width 0pt}p{2.8cm}<{\raggedright}|p{1.15cm}<{\centering}p{1.15cm}<{\centering}p{1.15cm}<{\centering}|p{1.15cm}<{\centering}p{1.15cm}<{\centering}p{1.15cm}<{\centering}|p{1.15cm}<{\centering}p{1.15cm}<{\centering}p{1.15cm}<{\centering}p{1.15cm}<{\centering}@{}}
\toprule
Method                 & \multicolumn{3}{c|}{CIFAR-10} & \multicolumn{3}{c|}{CIFAR-100} & \multicolumn{3}{c}{Tiny ImageNet} \\
\midrule\midrule
Full Data Training & \multicolumn{3}{c|}{95.50\%}         & \multicolumn{3}{c|}{77.28\%}          & \multicolumn{3}{c}{56.76\%}            \\
Budget ratio                & 10\%          & 20\%     & 30\%    & 10\%          & 20\%          & 30\%            & 10\%           & 20\%            & 30\%           \\
\midrule
Uniform & 92.06\% & 93.76\% & 94.61\% & 70.61\% & 74.18\% & 75.98\% & 48.36\% & 51.71\% & 53.76\% \\
Train Loss & 92.73\% & 93.87\% & 94.54\% & 65.12\% & 69.34\% & 72.62\% & 37.12\% & 45.23\% & 47.72\% \\
Grad Norm & 65.23\% & 76.23\% & 82.34\% & 64.72\% & 69.23\% & 72.34\% & 37.24\% & 44.34\% & 48.24\% \\
Grad Norm IS & 92.51\% & 93.78\% & 94.41\% & 69.34\% & 72.71\% & 73.21\% & 42.79\% & 47.34\% & 50.23\% \\
SVP & 57.38\% & 73.87\% & 82.34\% & 31.23\% & 43.35\% & 50.73\% & 19.34\% & 28.97\% & 34.24\% \\
Moderate-BS & 92.32\% & 93.57\% & 94.36\% & 70.21\% & 74.35\% & 75.34\% & 48.92\% & 51.36\% & 54.23\% \\
CCS-BS & 92.61\% & 93.88\% & 94.81\% & 71.11\% & 74.42\% & 76.21\% & 49.18\% & 52.43\% & 54.17\% \\
\rowcolor[HTML]{EFEFEF} 
\textbf{\method} & \textbf{94.65\%} & \textbf{94.83\%} & \textbf{95.07\%} & \textbf{73.11\%} & \textbf{76.10\%} & \textbf{77.21\%} & \textbf{50.84\%} & \textbf{55.03\%} & \textbf{55.94\%} \\
\bottomrule         
\end{tabular}
}
\vspace{-0.1in}
\end{table*}

\begin{table*}[t]
\centering
\caption{Epochs ($\downarrow$) required for \methodspace and various baseline methods to reach given target test accuracies on CIFAR-100 with different budget ratio 10\%, 20\%, 30\%. 
The target accuracies are set at 80\% and 90\% of the full dataset training accuracy (77.28\%), equivalent to 62\% and 69\%, respectively. \emph{NR} indicates that the target accuracy is not reached. Bold indicates the best results.}
 \vskip 0.02in
\label{tab:cifat100-epochs}
\resizebox{0.99\textwidth}{!}{
\begin{tabular}{@{}p{1.4cm}<{\centering}p{1.7cm}<{\centering}|p{1.3cm}<{\centering}cccp{1.3cm}<{\centering}cp{1.3cm}>
{\columncolor[HTML]{EFEFEF}\bfseries}p{1.3cm}<{\centering}!{\vrule width 0pt} @{}}
\toprule
Budget & Target Acc        & Uniform & Train Loss & Grad Norm & Grad Norm IS & SVP & Moderate-BS & CCS-BS & \method \\\midrule\midrule
\multirow{2}{*}{10\%}      & 62\% & 150 & 172 & 174 & 163 & \emph{NR} & 153 & 152 & 132 \\
                           &   69\%                &177 & \emph{NR} & \emph{NR} & 196 & \emph{NR} & 179 & 174 & 165 \\\midrule
\multirow{2}{*}{20\%}      & 62\% & 118 & 153 & 155 & 143 & \emph{NR} & 123 & 118 & 83 \\
                           &   69\%                &148 & 195 & 194 & 182 & \emph{NR} & 150 & 147 & 130 \\\midrule
\multirow{2}{*}{30\%}      & 62\%& 113 & 147 & 149 & 133 & \emph{NR} & 111 & 114 & 52 \\
                           &   69\%                & 146 & 175 & 174 & 170 & \emph{NR} & 148 & 145 & 111 \\\bottomrule 
\end{tabular}
}
\end{table*}

\textbf{Choice of selection features $U$.} For the features \(U = \{g(d_i,\theta)\}_{i=1}^{N_B}\) used in the selection, we leverage the gradient of each sample at the final layer. This choice is motivated by the following reasons: (1) Gradients directly capture the influence of data on model training; (2) Gradients are applicable to various supervised and unsupervised tasks without relying on specific task requirements or annotations; (3) The overhead of computing gradients at the final layer is negligible compared to the overall training cost of the batch.

\subsection{Discussion}
Existing reference-model-free batch selection methods independently score and select data in a sample-wise manner. Consequently, they cannot avoid selecting highly scored but mutually redundant samples, leading to a lack of diversity. 
In contrast, our \methodspace evaluates the selected subset as a whole and eliminates the impact of inter-sample redundancy, ensuring the diversity of the selected samples. Additionally, through theoretical analysis and approximation of the objective function, we provide an efficient selection process.

%% file: Tex/Experiments.tex
\section{Experiments}
\subsection{Experimental Setup}\label{sec:Experimental Setup}
\textbf{Datasets.} We conduct experiments to evaluate our \methodspace on CIFAR-10~\citep{krizhevsky2009learning}, CIFAR-100~\citep{krizhevsky2009learning}, and Tiny ImageNet~\citep{Le2015TinyIV} for image classification. We then conduct experiments in the context of class imabalance, specifically the CIFAR-100-LT dataset with imbalance ratio 100,  obtained through exponential sampling~\citep{DBLP:conf/cvpr/CuiJLSB19}. We further conduct experiments on more tasks, including semantic segmentation, cross-modal retrieval, and fine-tuning language models, using datasets PASCAL VOC 2012 trainaug~\citep{DBLP:conf/eccv/ChenZPSA18}, Wikipedia~\citep{DBLP:conf/mm/RasiwasiaPCDLLV10,DBLP:conf/cvpr/00020ZZ021}, and E2E NLG Challenge~\citep{DBLP:conf/sigdial/NovikovaDR17}.

\textbf{Baselines.} In addition to uniform sampling, we compared our \methodspace with various reference-model-free baseline methods, including training loss~\citep{DBLP:conf/aistats/KawaguchiL20}, gradient norm~\citep{DBLP:conf/icml/KatharopoulosF18}, gradient norm with importance sampling (gradient norm IS)~\citep{DBLP:conf/icml/KatharopoulosF18}, and Selection-via-Proxy (SVP)~\citep{DBLP:conf/iclr/ColemanYMMBLLZ20}. Additionally, we applied the selection strategies of Moderate~\citep{DBLP:conf/iclr/XiaL0S0L23} and CCS~\citep{DBLP:conf/iclr/ZhengLL023} in the online batch selection paradigm, referred to as Moderate-BS and CCS-BS. Both strategies have demonstrated superior performance in one-shot coreset selection with high pruning rates~\citep{DBLP:conf/iclr/ZhengLL023}.

\textbf{Implementation details.} For image classification, we use 18-layer ResNet as the backbone. The standard data augmentations
are applied as in \citet{DBLP:conf/nips/CubukZS020}. Models are trained
using SGD with momentum of 0.9 and weight decay of 0.005 as the optimizer. The initial learning rate is set to 0.1. We train the model for 200 epochs with the cosine learning-rate scheduling. Following \citet{pmlr-v162-mindermann22a,deng2023towards}, we set the budget of batch sample number as $N_S = 32$, and the budget ratio as $\frac{N_S}{N_B} = 10\%$ unless specified otherwise. Our implementations for semantic segmentation, cross-modal retrieval, and language model fine-tuning are aligned with the details in \citet{DBLP:conf/eccv/ChenZPSA18}, \citet{DBLP:conf/cvpr/00020ZZ021}, and \citet{DBLP:conf/iclr/HuSWALWWC22}, respectively.

\subsection{Performance Evaluation on Image Classification}\label{sec:exp:image classification}
We first empirically evaluate \methodspace on CIFAR-10, CIFAR-100 and Tiny ImageNet. We report the final accuracy of different methods with budget ratio $\frac{N_S}{N_B} = \{10\%,20\%,30\%\}$ in \cref{tab: classification-main}. We can observe that \methodspace significantly outperforms all baselines under different budgets across the three datasets. Furthermore, it is notable that no baseline consistently outperforms uniform sampling. Techniques like Train Loss and Grad Norm, which focus on selecting challenging samples, may exhibit decent performance in certain CIFAR-10 scenarios but suffer significant performance drops on more intricate datasets such as CIFAR-100 and Tiny ImageNet. Even with Moderate-BS and CCS-BS, which employ strategies like selecting samples with intermediate or diverse metrics, achieving results comparable to uniform sampling, they still notably lag behind \method. 

Moreover, \cref{tab:cifat100-epochs} illustrates the number of epochs needed to reach some given target accuracies. We can observe that some baselines like Grad Norm IS and CCS-BS that achieve slightly higher final accuracy compared to uniform sampling in \cref{tab: classification-main}, do not gain an advantage in terms of convergence speed.
Remarkably, \methodspace not only effectively boosts the final accuracy but also expedites the model's convergence.

\begin{table}[t]
\centering
\caption{Final accuracies of \methodspace and baselines and epochs required to reach given target accuracies on CIFAR-100-LT (imbalance ratio 100). The target accuracies are set at 60\% and 80\% of the full dataset training accuracy (42.80\%), equivalent to 26\% and 35\%, respectively. \emph{NR} indicates that the target Acc is not reached.
}
\label{tab:cifar-lt}
 \vskip 0.05in
\resizebox{0.99\linewidth}{!}{
\begin{tabular}{@{}cc|ccc!{\vrule width 0pt}@{}}
\toprule
\multicolumn{5}{c}{Full Data Training Acc: 42.80\%}                                                                         \\\midrule\midrule
\multirow{2}{*}{Budget} & \multirow{2}{*}{Method}  & \multirow{2}{*}{Final Acc $\uparrow$} & \multicolumn{2}{c}{Epochs to target Acc $\downarrow$} \\
                        &                          &                            &  26\%             &  35\%                  \\\midrule\midrule
\multirow{4}{*}{10\%}   & Uniform & 26.97\%   & 177 & \emph{NR} \\
&    Moderate-BS & 26.35\%   & 183 & \emph{NR} \\
&    CCS-BS& 27.43\%   & 178 & \emph{NR} \\
& \cellcolor[HTML]{EFEFEF}\textbf{\method}    & \cellcolor[HTML]{EFEFEF}\textbf{31.74\%}   & \cellcolor[HTML]{EFEFEF}\textbf{145} & \cellcolor[HTML]{EFEFEF}\emph{NR}\\\midrule
\multirow{4}{*}{20\%}   & Uniform & 38.50\%   & 95 & 144 \\
&    Moderate-BS   & 37.78\%   & 106 & 163 \\
&    CCS-BS  & 38.72\%   & 92 & 140 \\
& \cellcolor[HTML]{EFEFEF}\textbf{\method} & \cellcolor[HTML]{EFEFEF}\textbf{39.46\%}   & \cellcolor[HTML]{EFEFEF}\textbf{69} & \cellcolor[HTML]{EFEFEF}\textbf{118} \\\midrule
\multirow{4}{*}{30\%}   & Uniform                  & 39.67\%   & 54 & 113 \\
&    Moderate-BS   & 39.43\%   & 60 & 118 \\
&    CCS-BS  & 40.28\%   & 53 & 107 \\
& \cellcolor[HTML]{EFEFEF}\textbf{\method}  & \cellcolor[HTML]{EFEFEF}\textbf{42.41\%}   & \cellcolor[HTML]{EFEFEF}\textbf{48} & \cellcolor[HTML]{EFEFEF}\textbf{84} \\\bottomrule                 
\end{tabular}
}
\vskip -0.1in
\end{table}

\subsection{Performance Evaluation under Class Imbalance}\label{sec:exp:imb}
We further evaluate \methodspace on a more challenging task of imbalanced classification, a scenario often encountered in the real world~\citep{DBLP:conf/iclr/MenonJRJVK21,hong2023long,fan2022fedskip,fan2023federateda,FedLESAM}. Class imbalance poses a greater challenge to the diversity of data selection, as the already scarce samples from the tail classes are more likely to be completely absent in the selected subset.
Specifically, we experiment on CIFAR-100-LT (imbalance ratio 100)~\citep{DBLP:conf/cvpr/CuiJLSB19,DBLP:conf/icml/ZhouYWHZ22,DBLP:conf/nips/0002Y0ZHW23} with different budget ratios 10\%, 20\%, 30\%. We report the final accuracies and the numbers of epochs required to reach the target accuracies in \cref{tab:cifar-lt}.
\methodspace consistently outperforms baselines in both final performance and convergence speed, demonstrating its ability to ensure data diversity even under class imbalance.

\begin{table}[t]
\centering
\caption{Final mIoUs of \methodspace and various baseline methods and epochs required to reach given target mIoUs on  PASCAL VOC 2012 trainaug~\citep{DBLP:conf/eccv/ChenZPSA18} with different budget ratio 10\%, 20\%, 30\%. The target mIoUs are set at 80\% and 90\% of the full dataset training mIoU (70.80\%), equivalent to 57\% and 64\%, respectively. \emph{NR} indicates that the target accuracy is not reached.
}
 \vskip 0.05in
\label{tab:Segmentation}
\resizebox{0.99\linewidth}{!}{
\begin{tabular}{@{}cc|ccc!{\vrule width 0pt}@{}}
\toprule
\multicolumn{5}{c}{Full Data Training mIoU: 70.80\%}                                                                         \\\midrule\midrule
\multirow{2}{*}{Budget} & \multirow{2}{*}{Method}  & \multirow{2}{*}{Final mIoU $\uparrow$} & \multicolumn{2}{c}{Epochs to target mIoU $\downarrow$} \\
                        &                          &                            &  57\%             &  64\%                  \\\midrule\midrule
\multirow{4}{*}{10\%}   & Uniform & 63.72\%   & 20 & \emph{NR} \\
&    Moderate-BS & 63.27\%   & 23 & \emph{NR} \\
&    CCS-BS  & 63.98\%   & 21 & \emph{NR} \\
& \cellcolor[HTML]{EFEFEF}\textbf{\method}  & \cellcolor[HTML]{EFEFEF}\textbf{65.45\%}   & \cellcolor[HTML]{EFEFEF}\textbf{12} & \cellcolor[HTML]{EFEFEF}\textbf{38} \\\midrule
\multirow{4}{*}{20\%}   & Uniform & 67.07\%   & 8 & 25 \\
&    Moderate-BS & 66.83\%   & 10 & 33 \\
&    CCS-BS  & 67.22\%   & 8 & 26 \\
& \cellcolor[HTML]{EFEFEF}\textbf{\method} & \cellcolor[HTML]{EFEFEF}\textbf{68.13\%}   & \cellcolor[HTML]{EFEFEF}\textbf{7} & \cellcolor[HTML]{EFEFEF}\textbf{19} \\\midrule

\multirow{4}{*}{30\%}   & Uniform  & 68.56\%   & 8 & 19 \\
&    Moderate-BS & 68.34\%   & 9 & 22 \\
&    CCS-BS & 68.47\%   & 9 & 20 \\
& \cellcolor[HTML]{EFEFEF}\textbf{\method} & \cellcolor[HTML]{EFEFEF}\textbf{69.85\%}   & \cellcolor[HTML]{EFEFEF}\textbf{4} & \cellcolor[HTML]{EFEFEF}\textbf{13} \\\bottomrule                 
\end{tabular}
}
\vspace{-0.1in}
\end{table}

\begin{table}[t]
\centering
\caption{Mean Average Precision
(MAP) scores for both $\mathrm{image}\rightarrow\mathrm{text}$ and $\mathrm{text}\rightarrow\mathrm{image}$, along with their average on Wikipedia with different budget ratio 10\%, 20\%, 30\%.
Bold indicates the best results and \textcolor{darkred}{red} signifies improvements over full data training.}
 \vskip 0.05in
\label{tab:Retrieval}
\resizebox{0.99\linewidth}{!}{
\begin{tabular}{@{}cc|ccc!{\vrule width 0pt}@{}}
\toprule
Budget                & Methods & Img2Txt $\uparrow$ & Txt2Img $\uparrow$ & avg $\uparrow$ \\\midrule\midrule
100\%                 & Full Data Training    & 0.525   & 0.464   & 0.4945 \\\midrule
\multirow{2}{*}{10\%} & Uniform & 0.508   & 0.457   & 0.4825 \\
                      & \cellcolor[HTML]{EFEFEF}\textbf{\method}    & \cellcolor[HTML]{EFEFEF}\textbf{0.512}   & \cellcolor[HTML]{EFEFEF}\textbf{0.462}   & \cellcolor[HTML]{EFEFEF}\textbf{0.487}  \\\midrule
\multirow{2}{*}{20\%} & Uniform & 0.513   & 0.462   & 0.4875 \\
                      & \cellcolor[HTML]{EFEFEF}\textbf{\method}    & \cellcolor[HTML]{EFEFEF}\textbf{0.517}   & \cellcolor[HTML]{EFEFEF}\textbf{\textcolor{darkred}{0.467}}   & \cellcolor[HTML]{EFEFEF}\textbf{0.492}  \\\midrule
\multirow{2}{*}{30\%} & Uniform & 0.518   & 0.464   & 0.491  \\
                      & \cellcolor[HTML]{EFEFEF}\textbf{\method}    & \cellcolor[HTML]{EFEFEF}\textbf{0.522}   & \cellcolor[HTML]{EFEFEF}\textbf{\textcolor{darkred}{0.468}}   & \cellcolor[HTML]{EFEFEF}\textbf{\textcolor{darkred}{0.495}}
                      \\\bottomrule 
\end{tabular}
}
\end{table}

\begin{figure*}[t]
    \centering
    \vspace{0.2in}
\subfigure[KNN distance]{   
\begin{minipage}{0.31\textwidth}
\centering    
\vspace{-0.1in}
\label{fig:knn}
\includegraphics[width=\textwidth]{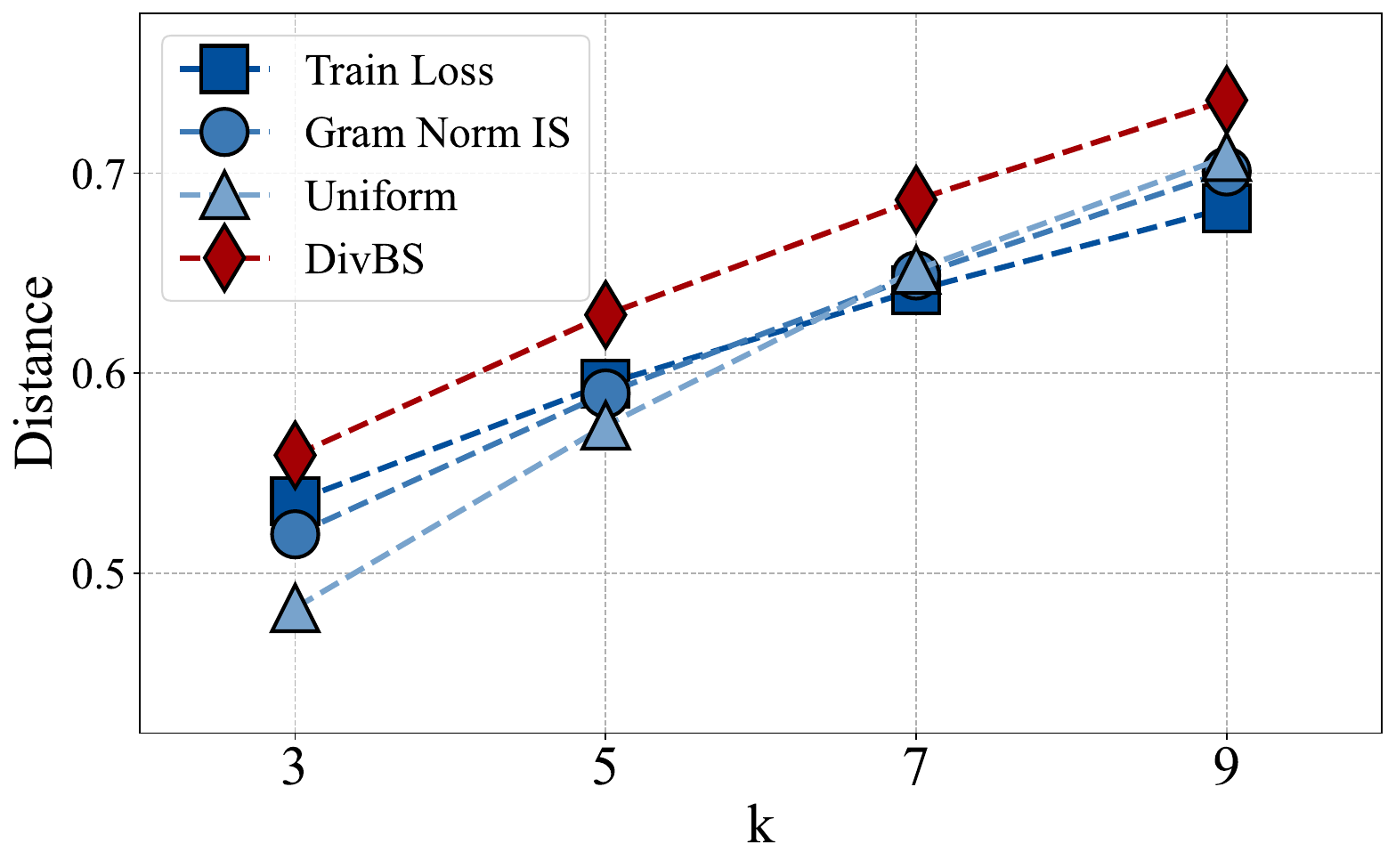}  
\end{minipage}}
\subfigure[Group Proportion]{   
\begin{minipage}{0.31\textwidth}
\centering    
\vspace{-0.1in}
\label{fig:property}
\includegraphics[width=\textwidth]{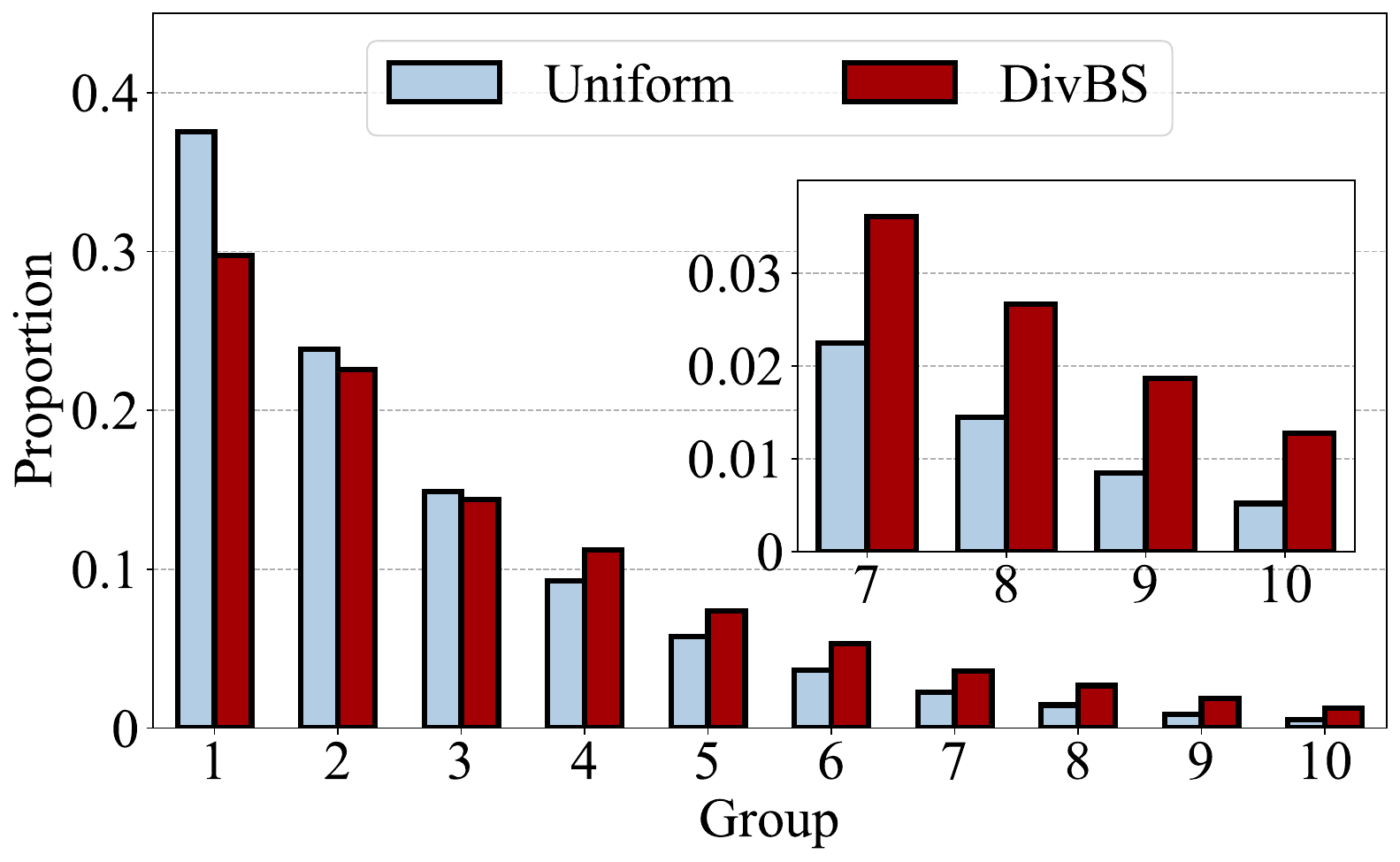}
\end{minipage}
}
\subfigure[CIFAR-10/100*]{   
\begin{minipage}{0.333\textwidth}
\centering   
\vspace{-0.17in}
\label{fig:cifar*}
\includegraphics[width=\textwidth]{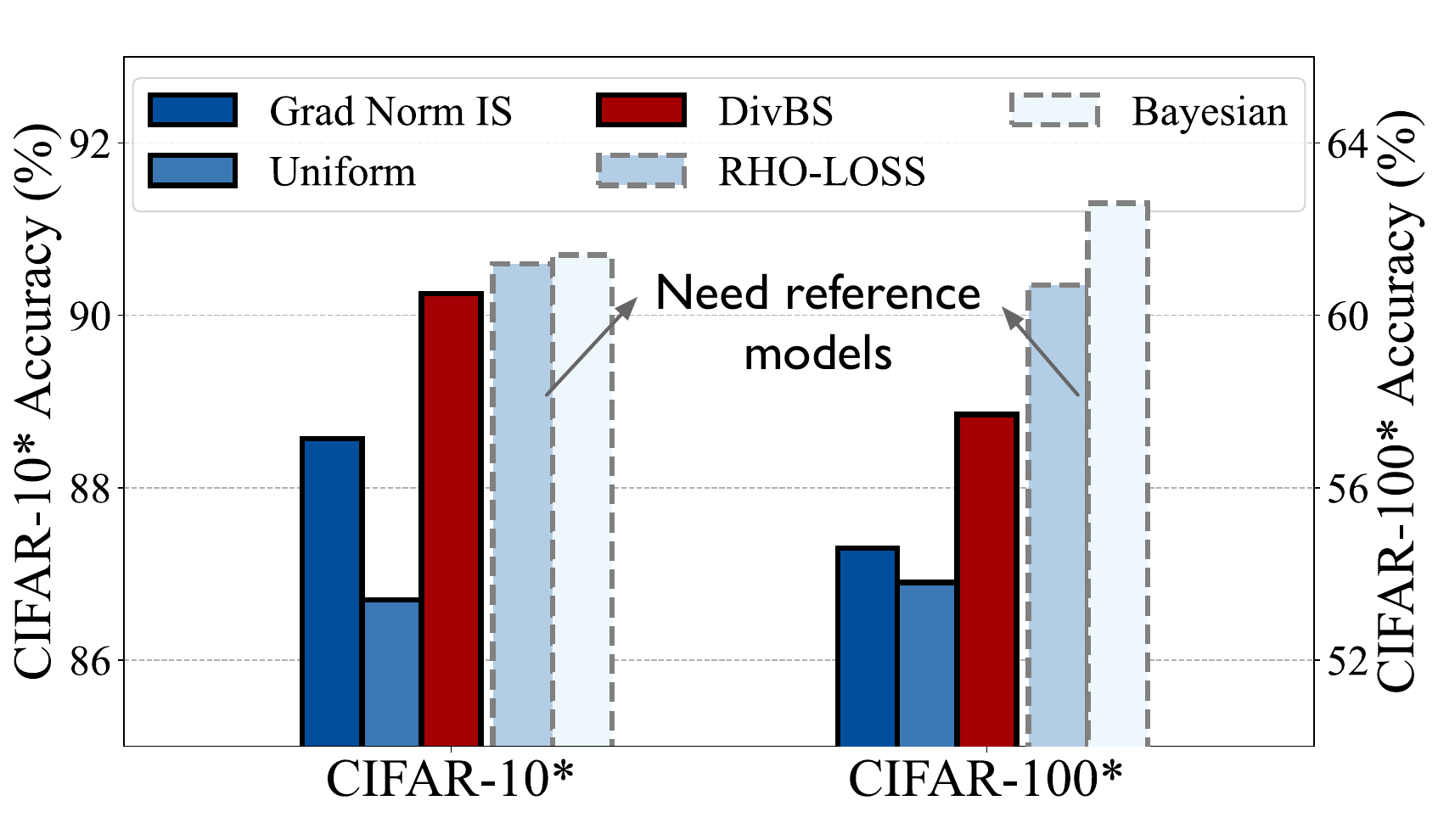}\vspace{1mm}
\end{minipage}
}
\caption{(a) Average mean feature cosine distance with the k-nearest neighbors for the selected data on CIFAR-10 (10\% budget). (b) Properties of 10 groups in the selected data on CIFAR-100-LT. (c) Performance comparison on CIFAR-10* and CIFAR-100*. Note that Beyesian and RHO-LOSS requring reference models and also introduce additional overhead from using auxiliary models for inference.}
\vspace{-0.05in}
\end{figure*}

\begin{figure*}[t]
 \vskip 0.2in
\centering 
\begin{minipage}{0.315\textwidth}
\centering   
\includegraphics[width=\textwidth]{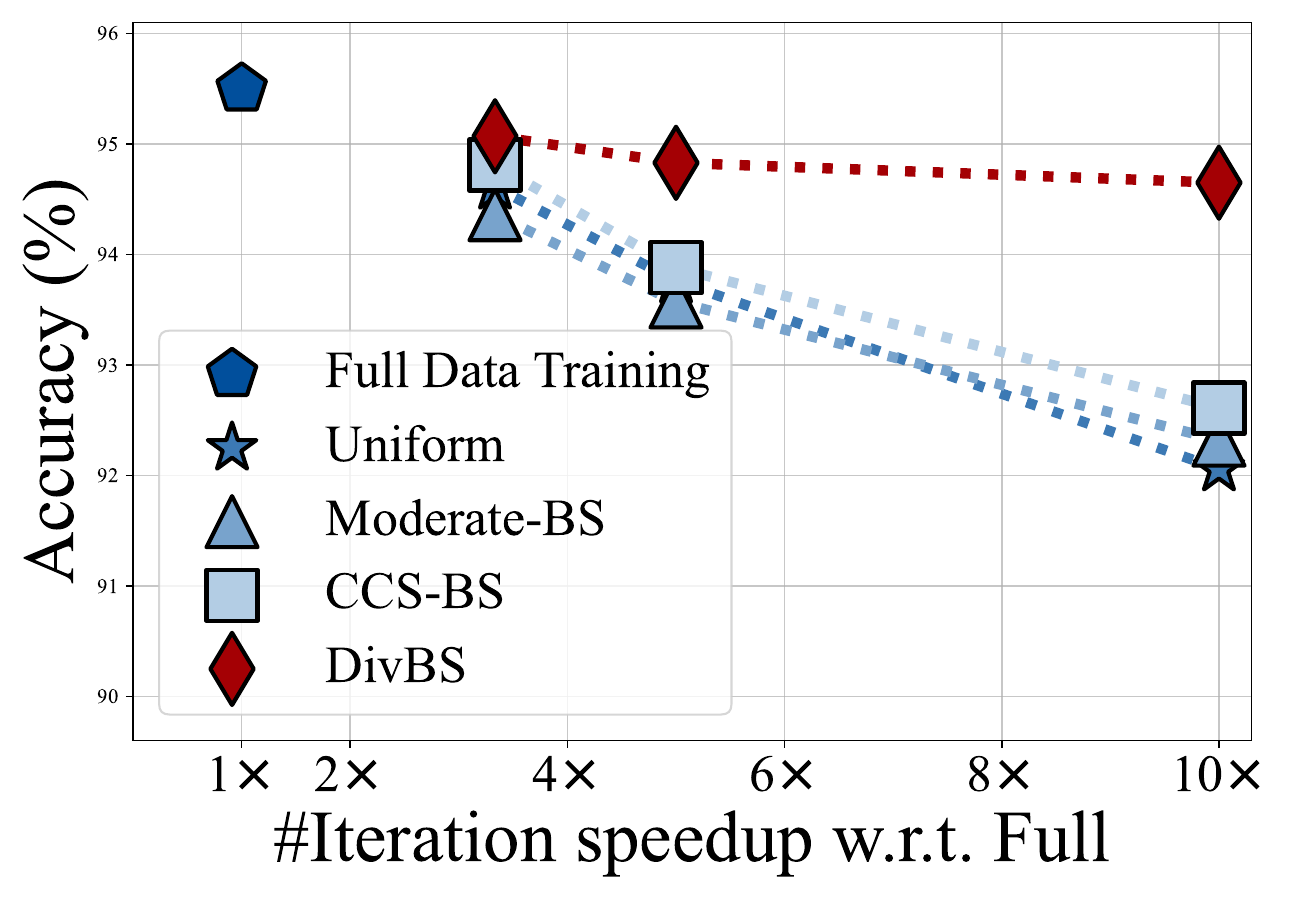}
\end{minipage}
\hspace{1mm}
\begin{minipage}{0.328\textwidth}
\centering   
\includegraphics[width=\textwidth]{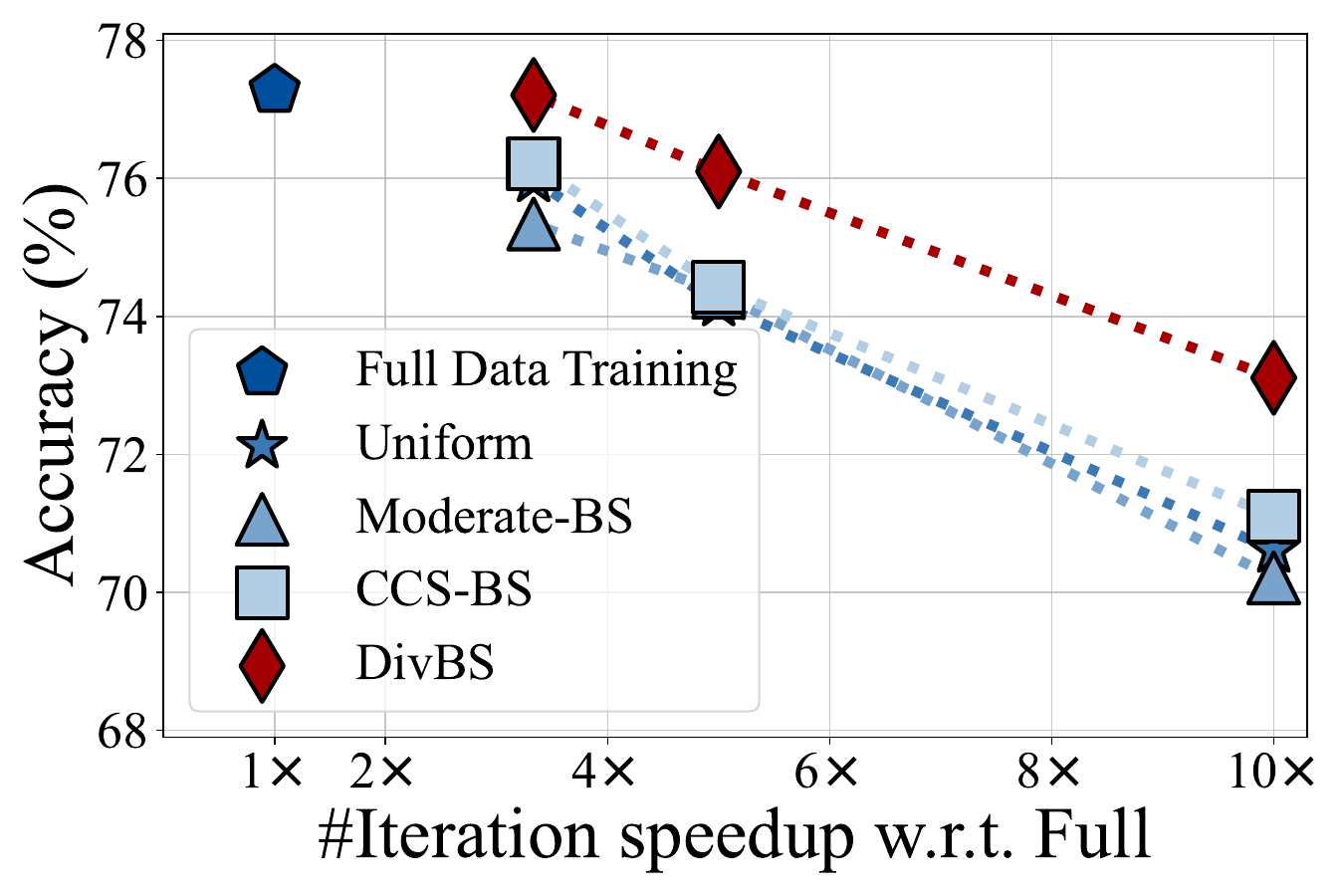}
\end{minipage}
\hspace{1mm}
\begin{minipage}{0.315\textwidth}
\centering   
\includegraphics[width=\textwidth]{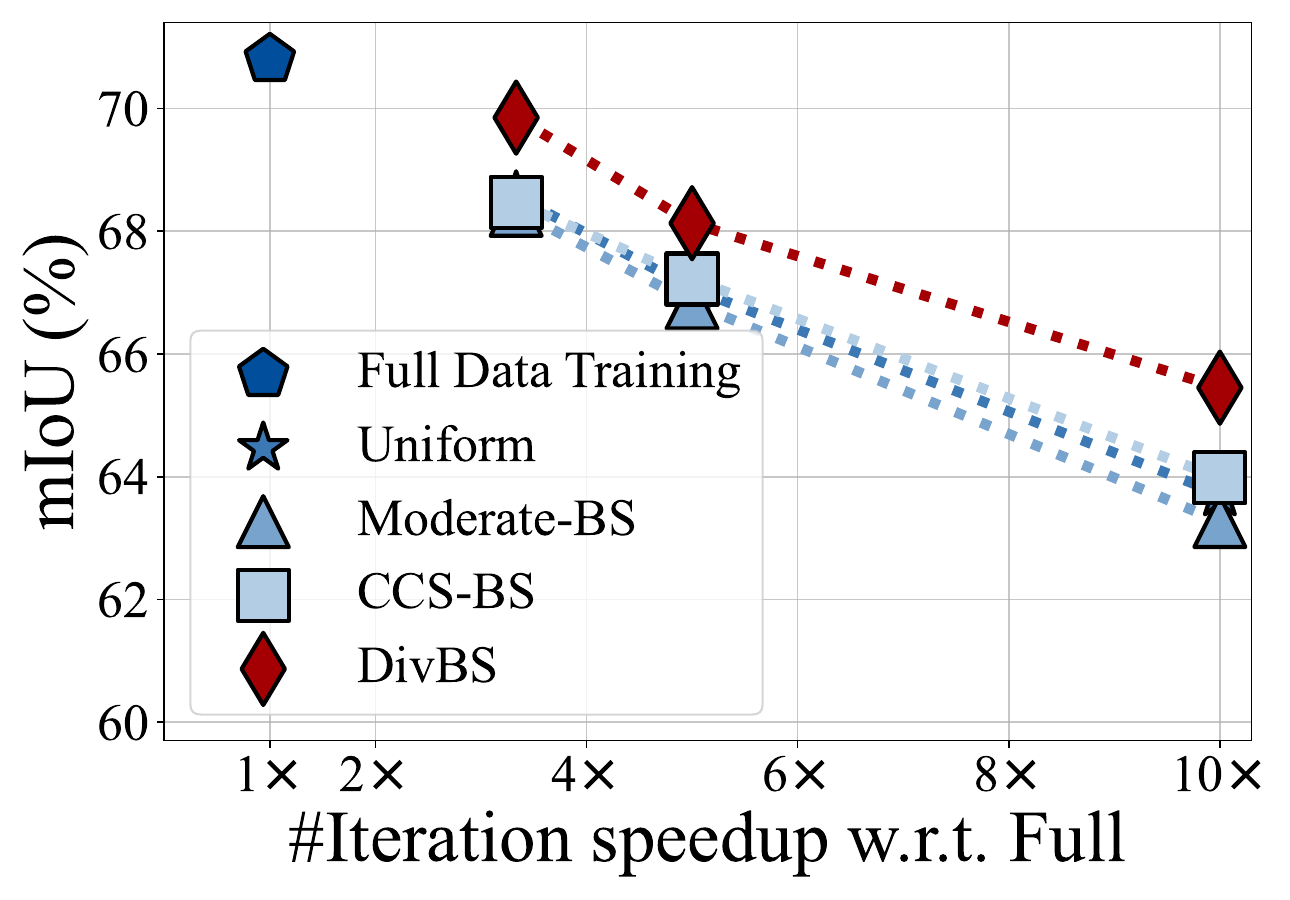}
\end{minipage}
\\
\subfigure[CIFAR-10]{  
\begin{minipage}{0.315\textwidth}
\centering    
\includegraphics[width=\textwidth]{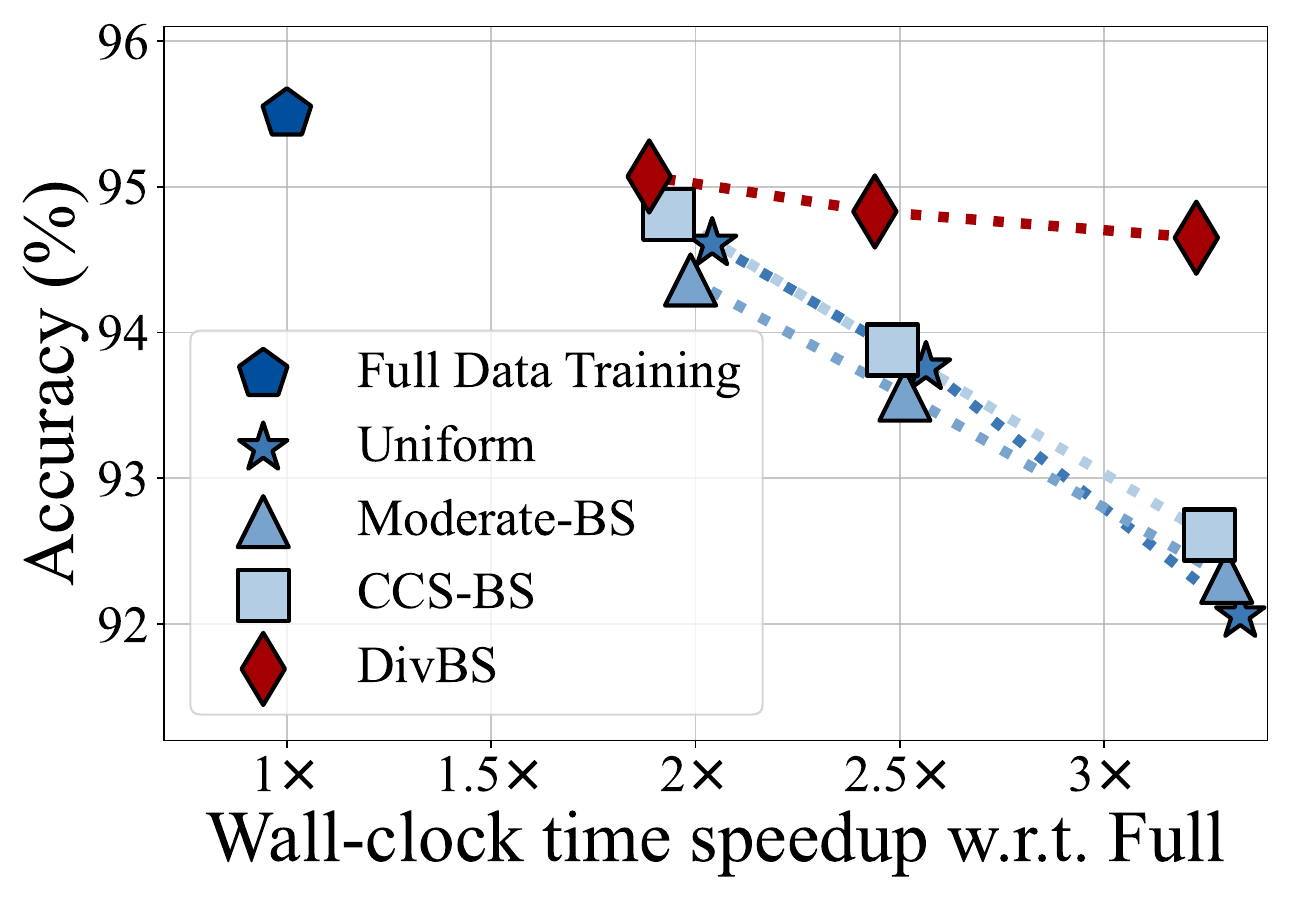}
\end{minipage}
}
\hfill
\subfigure[CIFAR-100]{
\begin{minipage}{0.32\textwidth}
\centering    
\includegraphics[width=\textwidth]{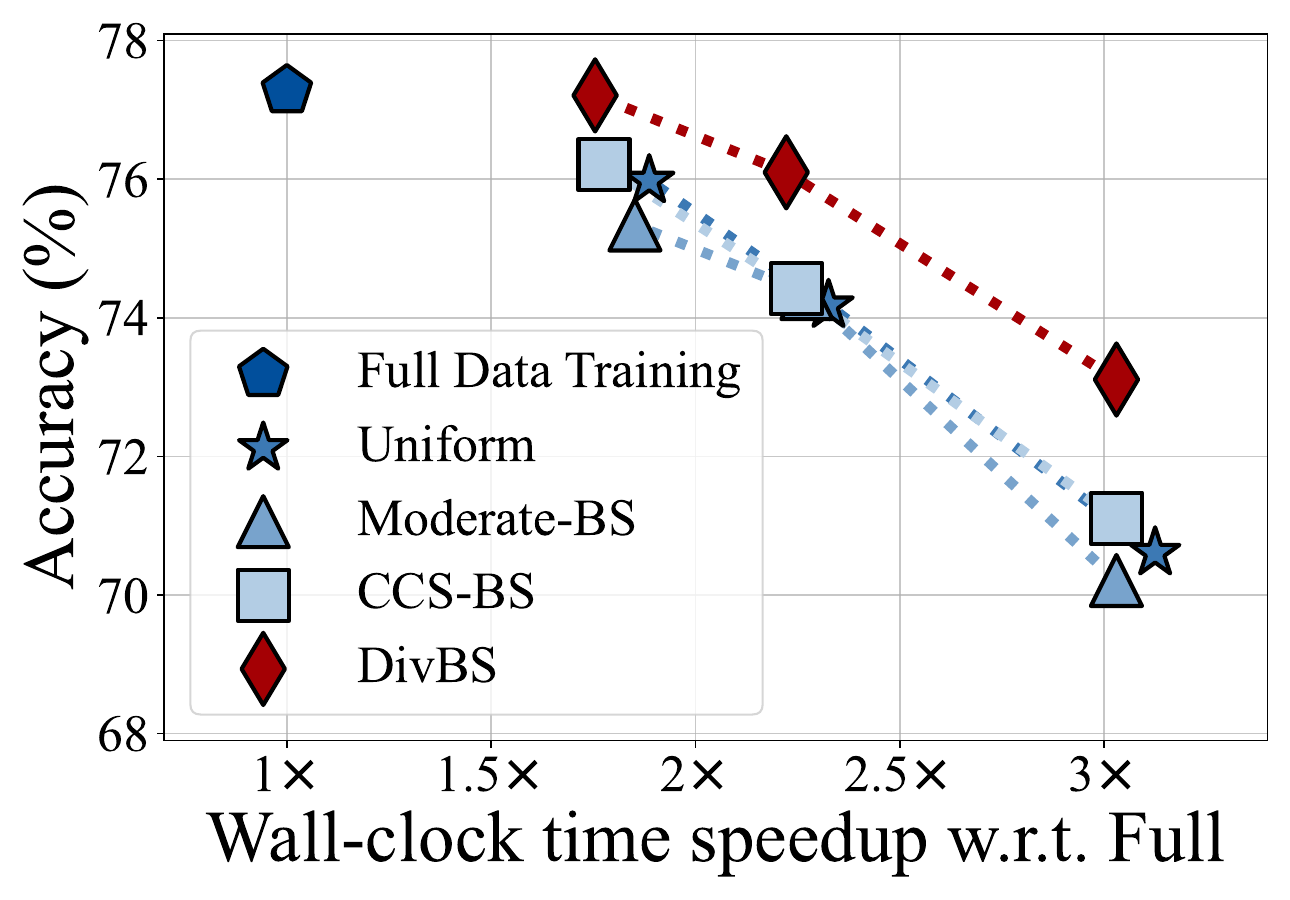}
\end{minipage}
}
\hfill
\subfigure[PASCAL VOC 2012 trainaug]{   
\begin{minipage}{0.315\textwidth}
\centering    
\includegraphics[width=\textwidth]{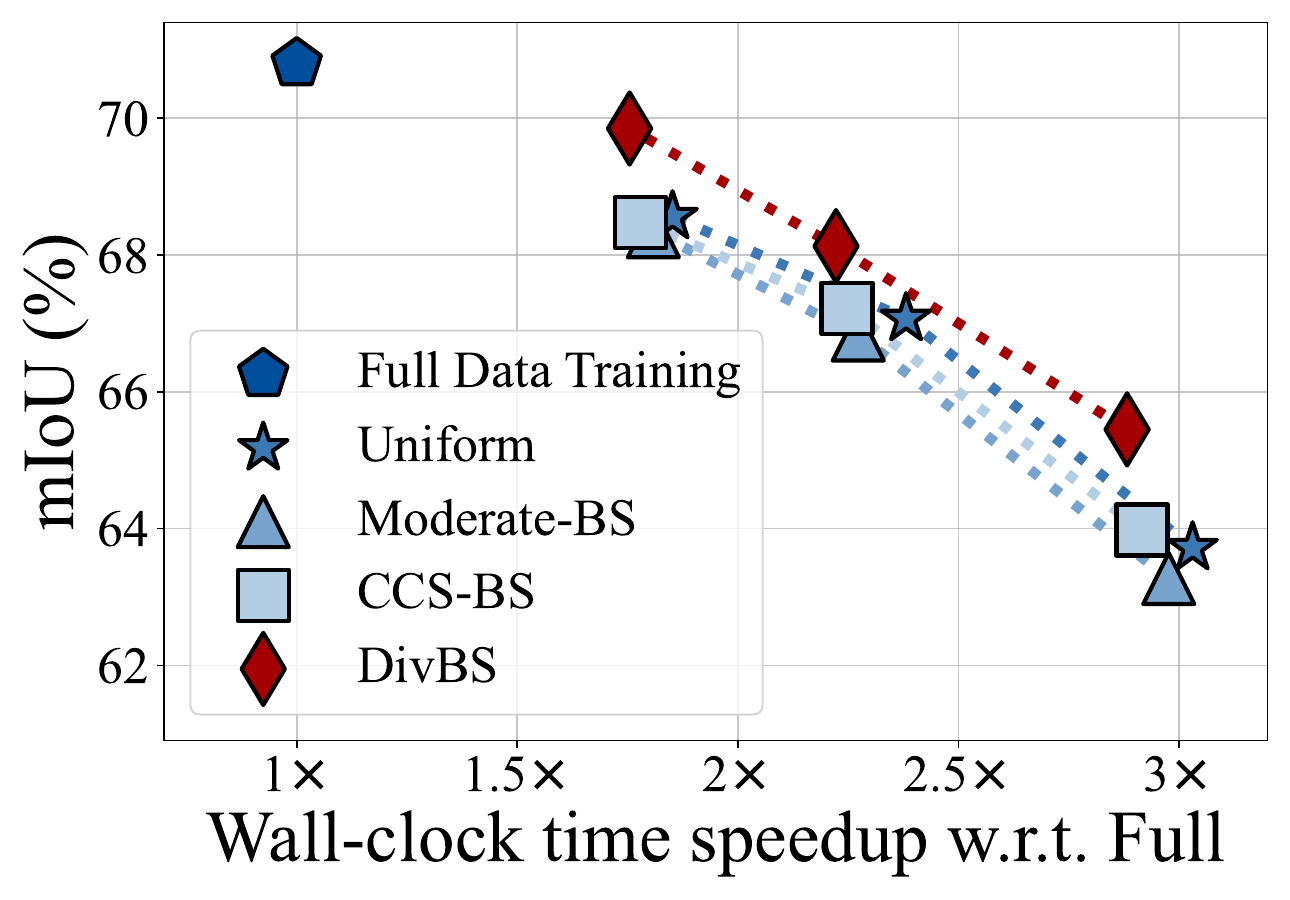}  
\end{minipage}
}
\caption{Performance ($\uparrow$) \emph{v.s.} speedup ($\uparrow$) on (a) CIFAR-10, (b) CIFAR-100, and (c) PASCAL VOC 2012 trainaug. 
The \textbf{upper} panel displays the relationship between the performance (accuracy or mIoU) of different methods and the speedup \emph{w.r.t.} the number of training iterations. The \textbf{lower} panel illustrates the relationship between the performance and the speedup \emph{w.r.t.} the wall-clock time. 
}  
\label{fig: performance-vs-speedup}
\vskip -0.1in 
\end{figure*}

\subsection{Performance Evaluation on Semantic Segmentation}
We compare the performance of different methods on PASCAL VOC
2012 trainaug dataset for semantic segmentation, which is a crucial and practical dense prediction task. Specifically, we train the DeepLabV3~\citep{DBLP:journals/corr/ChenPSA17} segmentation model with MobileNet~\citep{DBLP:journals/corr/HowardZCKWWAA17} as the backbone for 50 epochs, aligning with the details provided in \citet{DBLP:conf/eccv/ChenZPSA18}. For Moderate-BS and CCS-BS, their original metrics, based on classification concepts, are not directly applicable to segmentation tasks. We replace them with training losses that similarly characterize the difficulty of samples, while retaining their selection strategies. We report the results in \cref{tab:Segmentation}. Still, our proposed \methodspace outperforms uniform sampling and other baseline methods in aspects of both the speed to reach the target mIoUs and the final mIoU for the semantic segmentation task.
\subsection{Performance Evaluation on Cross-Modal Retrieval}
In addition to visual tasks, we extend our empirical research to evaluate \methodspace in the cross-modal retrieval task. Specifically, we conduct experiments on Wikipedia with different
budget ratio 10\%, 20\%, 30\%. 
We employ an ImageNet-pretrained VGG-19 model~\citep{DBLP:journals/corr/SimonyanZ14a} as the image backbone and a pre-trained Doc2Vec model~\citep{DBLP:conf/rep4nlp/LauB16} for text. All other implementation details align with \citet{DBLP:conf/cvpr/00020ZZ021}.
In \cref{tab:Retrieval}, we report the Mean Average Precision (MAP) score for both $\mathrm{image}\rightarrow\mathrm{text}$ and $\mathrm{text}\rightarrow\mathrm{image}$ retrieval, along with their average. It can be observed that our \methodspace consistently outperforms uniform sampling in the cross-modal retrieval task. More impressively, \methodspace even surpasses full-data training in terms of $\mathrm{text}\rightarrow\mathrm{image}$ MAP and average MAP at 20\% and 30\% budget ratios. This indicates that our \methodspace remains effective in selecting diverse and high-quality subsets in the cross-modal retrieval task.

\begin{table}[t]
\centering
\caption{Performance of GPT-2 medium (M) with LoRA finetuning on the E2E NLG Challenge with budget ratio 20\%.}
\label{tab:GPT2}
\resizebox{0.99\linewidth}{!}{
\begin{tabular}{c|c|c>{\columncolor[HTML]{EFEFEF}\bfseries}c}
\toprule
GPT-2 M (LoRA) & Full Data Training & Uniform & \method \\
\midrule\midrule
BLEU $\uparrow$          & 68.07\%     &   66.14\%      &   66.87\%                     \\
NIST  $\uparrow$         & 8.726      &  8.624       &  8.685                      \\
MET $\uparrow$           & 46.43\%     &  44.26\%       &  44.57\%                      \\
ROUGE-L $\uparrow$       & 69.44\%     &  66.90\%       &  67.83\%                      \\
CIDEr $\uparrow$         & 2.444     &  2.301       &  2.343   \\\bottomrule                  
\end{tabular}
}
\end{table}

\begin{table}[t]
\centering
\caption{Wall-clock time ($\downarrow$) per epoch of different methods on CIFAR-10 and PASCAL VOC 2012 trainaug.
The results are averaged over ten epochs. The subscript indicates the percentage of time saved compared to full data training.}
\label{tab:wall-clock}
\resizebox{0.97\linewidth}{!}{
\begin{tabular}{c|c|c>{\columncolor[HTML]{EFEFEF}}c}
\toprule
\multicolumn{4}{c}{CIFAR-10} \\\midrule\midrule
Full data training                  & Budget ratio     & Uniform    & \textbf{\method}       \\\midrule\midrule
\multirow{3}{*}{44.7s} & 10\% & 13.2s \textsubscript{$\downarrow$70\%} & 13.9s \textsubscript{$\downarrow$69\%} \\
                       & 20\% & 17.4s \textsubscript{$\downarrow$61\%} & 18.5s \textsubscript{$\downarrow$59\%} \\
                       & 30\% & 22.1s \textsubscript{$\downarrow$51\%} & 23.6s \textsubscript{$\downarrow$47\%}\\\midrule\midrule
\multicolumn{4}{c}{PASCAL VOC 2012 trainaug} \\\midrule\midrule
Full data training      & Budget ratio  & Uniform     & \textbf{\method}        \\\midrule\midrule
\multirow{3}{*}{291.2s} & 10\%   & 96.1s \textsubscript{$\downarrow$67\%}  & 100.9s \textsubscript{$\downarrow$65\%} \\
                        & 20\%   & 125.3s \textsubscript{$\downarrow$58\%} & 132.4s \textsubscript{$\downarrow$55\%} \\
                        & 30\%   & 157.2s \textsubscript{$\downarrow$46\%} & 166.9s \textsubscript{$\downarrow$43\%}\\
\bottomrule
\end{tabular}
}
\vspace{-0.1in}
\end{table}

\subsection{Performance Evaluation on LM Finetuning}
We also validate the effectiveness of our \methodspace on the language model finetuning task. Specifically, we finetune the GPT-2 Medium (M) model~\citep{radford2019language} using LoRA~\citep{DBLP:conf/iclr/HuSWALWWC22} on the E2E NLG Challenge~\citep{DBLP:conf/sigdial/NovikovaDR17}, which is widely used dataset for natural language generation evaluations. We finetune GPT-2 M for 5 epochs with a minibatch size of 4, aligning with the remaining implementation details in \citet{DBLP:conf/iclr/HuSWALWWC22}. We report the results of full data training, uniform sampling and our \methodspace with budget ratio 20\% in \cref{tab:GPT2}. We can observe that our \methodspace consistently outperforms uniform sampling across all 5 metrics, further demonstrating its universality across different tasks and training paradigms.

\subsection{Further Analysis}
\textbf{Properties of the Selected Data.} In 
\cref{fig:knn}, we compare the average feature cosine distances to the k-nearest neighbors ($k=1, 3, 5, 7, 9$) of samples selected by different methods on CIFAR-10. 
\methodspace stands out with the largest KNN distances, highlighting the reduced redundancy and broader coverage of the selected samples.
We then arrange the classes of CIFAR-100-LT in descending order of sample number, grouping every ten classes together, denoted as groups 1-10. In \cref{fig:property}, we illustrate the proportion of samples selected by uniform sampling and our \methodspace for each group. It is evident that \methodspace consistently increases the proportion of tail samples, demonstrating the effective enhancement of diversity in the selected subset.

\begin{table}[t]
\centering
\caption{Final accuracies ($\uparrow$) on CIFAR-100 with different budget ratio 10\%, 20\%, 30\% when using SGD and AdamW as optimizer.}
\vskip 0.03in
\label{tab:optimizer}
\resizebox{0.97\linewidth}{!}{
\begin{tabular}{c|c|p{1.6cm}<{\centering}>{\columncolor[HTML]{EFEFEF}\bfseries}p{1.6cm}<{\centering}}
\toprule
\multicolumn{4}{c}{SGD} \\\midrule\midrule
Full data training                  & Budget ratio    & Uniform    & \method       \\\midrule\midrule
\multirow{3}{*}{77.28\%} & 10\% &70.61\%  & 73.11\% \\
                       & 20\% &74.18\%  & 76.10\% \\
                       & 30\% &75.98\%  & 77.21\%\\\midrule\midrule
\multicolumn{4}{c}{AdamW} \\\midrule\midrule
Full data training      & Budget ratio & Uniform     & \method        \\\midrule\midrule
\multirow{3}{*}{70.50\%} & 10\%   & 66.72\%	& 69.07\% \\
                        & 20\%   & 68.93\% &	70.43\% \\
                        & 30\%   & 69.09\% &	70.47\% \\
\bottomrule
\end{tabular}
}
\end{table}

\textbf{Wall-clock time.} Compared to uniform sampling, our \methodspace incurs some additional overhead in the selection process. While our research primarily investigates data selection strategies favorable for model convergence, we still empirically measure and compare the practical impact of different methods on training duration. In \cref{tab:wall-clock}, we report the wall-clock time per epoch on CIFAR-10 image classification and PASCAL VOC 2012 trainaug semantic segmentation. We can observe that the time proportion of uniform sampling compared to full time training is greater than the corresponding budget ratio. This discrepancy arises because batch selection only reduces the data used for network updates while operations like loading data, model validation, and saving model files still require the same amount of time. Our \methodspace introduces additional overhead of less than 5\% total time compared to uniform sampling. There is potential to further reduce this overhead using hardware techniques or parallelization methods, as discussed in \cref{sec:Acceleration of the Selection Process}.

\textbf{Robustness with different optimizers.} We validate the robustness of our \methodspace under different optimizers. \cref{tab:optimizer} showcases the performance of our \methodspace on CIFAR-100 using both SGD and AdamW.
\methodspace consistently outperforms the baseline across various budget ratios. Moreover, with both optimizers, \methodspace exhibits minimal performance loss at 30\% budget compared to full dataset training.

\textbf{Narrow the gap with methods involving extra reference models.} In \cref{fig:cifar*}, we compare our \methodspace with RHO-LOSS~\citep{pmlr-v162-mindermann22a} and Bayesian~\citep{deng2023towards}, which utilize extra reference models for selection, on CIFAR-10* and CIFAR-100* with 10\% budget. The implementation details are strictly aligned with those of RHO-LOSS and Bayesian. CIFAR-10/100* are versions of CIFAR-10/100, with only half of the data retained~\citep{pmlr-v162-mindermann22a}.  Our \methodspace significantly narrows the gap with methods that utilize extra reference models.

\textbf{Trade-off between performance and speedup.} 
While our primary aim is to reduce training costs while preserving performance, there inherently exists a trade-off between model performance and acceleration effects. In \cref{fig: performance-vs-speedup}, we present the trade-off between the performance and speedup (\emph{w.r.t.} training iterations and wall-clock time) of various methods on CIFAR-10, CIFAR-100, and PASCAL VOC 2012 trainaug datasets. Points located in the upper-right corner of the subfigures indicate superior performance coupled with enhanced acceleration effects. Notably, our \methodspace excels in achieving a superior performance-speedup trade-off.

%% file: Tex/RelatedWork.tex
\section{Related Work}\label{sec:related work}
\textbf{Coreset selection}~\citep{DBLP:conf/icml/MirzasoleimanBL20,zhang2024TDDS}, also known as data pruning, aims to create a smaller subset (coreset) of the original data that captures essential patterns for efficient model training.  Various metrics like the entropy score~\citep{DBLP:conf/iclr/ColemanYMMBLLZ20}, EL2N score~\citep{DBLP:conf/nips/PaulGD21}, forgetting score~\citep{DBLP:conf/iclr/TonevaSCTBG19}, and classification margin~\citep{DBLP:conf/nips/Pleiss0EW20}, are used to measure individual differences among data points.
Yet, selecting samples with the highest scores can lead to diversity issues, especially at high pruning rates, resulting in performance degradation~\citep{DBLP:conf/iclr/XiaL0S0L23}. \citet{DBLP:conf/iclr/ZhengLL023,DBLP:conf/iclr/XiaL0S0L23} propose strategies of selecting samples with intermediate scores or with diverse scores, yielding promising results under high pruning rates.~However, coreset selection faces limitations in prioritizing samples with diverse properties during different training stages. Moreover, the acceleration benefits are noticeable only when the coreset is repeatedly used to train various models, as the data selection relies on a full data trained model.

\textbf{Curriculum learning}~\citep{bengio2009curriculum} seeks to enhance model performance with minimal computational costs by  prioritizing ``easy" samples before uniformly training on the entire dataset~\citep{jiang2015self,sinha2020curriculum,zhou2020uncertainty}. Although curriculum learning can improve model convergence, it may not efficiently reduce training expenses. And they fall short in addressing the challenge of skipping redundant points that have already been learned.

\textbf{Online batch selection} speeds up model training by using only a portion of data in each batch. \citet{Jiang2019AcceleratingDL,pmlr-v80-katharopoulos18a,Loshchilov2015OnlineBS} prioritize hard samples based on criteria like training loss or gradient norm, but these can hinder early-stage model convergence and be sensitive to outliers. \citet{pmlr-v162-mindermann22a} and \citet{deng2023towards} achieve notable speedup by leveraging additional reference models to select valuable samples. However, their practical applications are restricted by the availability of well-performing reference models. Compared to prior methods, which score and select data in a sample-wise manner, our reference-model-free \method, excels in selecting high-quality and diverse samples by optimizing the overall orthogonalized representativeness of the subset after removing inter-sample redundancy.

\label{sec:Acceleration of the Selection Process}
\vspace{-0.03in}
\textbf{Acceleration of the Selection Process.} Except for uniform sampling, online batch selection methods generally require an additional forward pass for each batch. \citet{jouppi2017datacenter} achieve nearly 10$\times$ acceleration in forward propagation by leveraging low-precision cores on GPUs or TPUs, grounded in the observation that forward propagation exhibits higher tolerance to low precision. \citet{alain2015variance} utilize a group of workers to asynchronously execute forward propagation and selection while the main process trains on the recently chosen data, thereby saving time of selection processes. Selection can also be cheaper by reusing features, gradients, or losses computed in previous epochs~\citep{Loshchilov2015OnlineBS}. 
Though our research scope is limited to the effects of different selection strategies, exploring the integration of these techniques for maximum acceleration is a promising and noteworthy avenue.

%% file: Tex/Conclusion.tex
\section{Conclusion}
We investigate the diversity challenge that may arise from existing batch selection methods independently scoring and selecting data in a sample-wise manner. We propose \Methodspace (\method), which selects diverse and representative subsets by optimizing the overall orthogonalized representativeness after removing inter-sample redundancy, thereby accelerating model training. Extensive experiments validate the superiority of our \methodspace in performance-speedup tradeoff across various tasks.

\section*{Acknowledgement}
This work is supported by the National Key R\&D Program of China (No. 2022ZD0160702), STCSM (No. 22511106101, No. 22511105700, No. 21DZ1100100), 111 plan (No. BP0719010) and National Natural Science Foundation of China (No. 62306178). LY is supported by Career Development Fund (CDF) of the Agency for Science, Technology and Research (A*STAR) (No: C233312007). 

\section*{Impact Statement}
Investigating the impact of data selection on the performance of minority group data is of paramount importance, especially for some social applications, such as medical diagnosis, auto-driving, and criminal justice. Some data selection algorithms may prioritize minority groups, due to their slower learning pace and the challenging nature of instances within these groups~\citep{DBLP:conf/icml/YangZKG23,hong2024on}. Conversely, other algorithms may downplay the importance of rare groups, as excluding them has a minimal impact on the overall performance of the training set. Our method is designed to guarantee the diversity of selected data, thereby somewhat safeguarding the priority of minority group data within the selection process, as demonstrated in experiments under class imbalance.

%% file: Tex/Appendix.tex
\newpage
\appendix
\onecolumn

\section{Notations}
In \cref{tab:notations}, we summarize the notations used in this paper.
\begin{table}[H]
    \centering
     \caption{Summary of notations}\label{tab:notations}
    \begin{tabular}{p{3cm}<{\raggedright}|p{5cm}<{\raggedright}|p{7cm}<{\raggedright}}
    \toprule
     Category    & Notation & Description  \\
     \midrule\midrule
       Data and Sets  & $\gD$ & Training data set\\
       &$B$ & (Large) training data batch\\
       &$N_B$ & Size of $B$\\
       & $d$ & a data point\\
       & $S\subset B$ & A smaller subset of $B$\\
       & $N_S$ & Budget of the size of $S$ \\
       & $U = g(B,\theta) = \{g(d_i,\theta)\}_{i=1}^{N_B}$ & Features used for selection from $B$ \\
      & $u\in U$ & A feature in $U$ \\
      &$\gE(g(S,\theta))$    & $\{E = \{e_i,\ldots e_{|E|}\} | \forall e_i, e_j \in E, e_i\cdot e_j = \delta_{ij}, \mathrm{span}(E) = \mathrm{span}(g(S,\theta))\}$, Set of all potential orthonormal bases for the subspace spanned by $g(S,\theta)$  \\
      & $E\in \gE(g(S,\theta))$ & A orthonormal base for $g(S,\theta)$  \\
      & $e\in E$ & A unit vector in $E$\\
      &$j,k \in B \setminus S$ & Two data points not in $S$   \\
      
\midrule
       Model and functions &$f_{\theta}$ & a deep model with parameters $\theta$ \\
       & $\theta$ & Model parameters\\
       & $g(B,\theta)$, $g(d,\theta)$ & Mapping function from data points to the selection features, given model parameters $\theta$ \\
       & $s(u)$ & Scoring function for sample-wise selection \\
       &$r(S,B,\theta)$& $\max_{E\in \gE(g(S,\theta))} \sum_{e\in E}\sum_{u\in g(B,\theta)}e\cdot u$, the orthogonalized representativeness of subset \(S\) with respect to \(B\)\\
       &\(\delta_{ij}\)    & Kronecker delta, taking the value $1$ when \(i = j\) and $0$ otherwise\\
       &$\mathrm{span}(\cdot)$ & Subspace spanned by all elements in a set \\
       &$r'(S,B,\theta)$& $\sqrt{\sum_{e\in E} (e\cdot \sum_{u\in g(B,\theta)}u)^2 }$, $\forall E \in \mathcal{G}_E(g(S,\theta))$, an auxiliary function introduced to study algorithm performance\\
       &$F(\cdot)$    & A general set function  \\
       &$\hat{\beta}(x)$ & $\sqrt{a+x}-\sqrt{x}$ with $a>0$ \\\midrule
       Others &$e$ & Euler's Number, approximately equal to $2.71828$  \\\bottomrule
    \end{tabular}
    
\end{table}
\clearpage

\section{Technical Details}
\subsection{Proof of \cref{prop:basis}}\label{appendix:Proof of prop:basis}
\begin{proof}[Proof of \cref{prop:basis}]\ \\
Based on \cref{eq:representativeness-subset}:
\begin{equation}
    \begin{split}
        r(S,B,\theta)&=\max_{E\in \gE(g(S,\theta))} \sum_{e\in E}\sum_{u\in g(B,\theta)}e\cdot u \\
        &=\max_{E\in \gE(g(S,\theta))} (\sum_{e\in E}e)\cdot(\sum_{u\in g(B,\theta)}u) \\
        &= \max_{E\in \gE(g(S,\theta))} (\sum_{e\in E}e)\cdot(\sum_{e\in E} (e\cdot \sum_{u\in g(B,\theta)}u)e)
    \end{split}
\end{equation}
where $\sum_{e\in E} (e\cdot \sum_{u\in g(B,\theta)}u)e$ is the projection of $\sum_{u\in g(B,\theta)}u$ onto the subspace spanned by $g(S,\theta)$, and it remains constant with variations of $E$, with a length of $\sqrt{\sum_{e\in E} (e\cdot \sum_{u\in g(B,\theta)}u)^2 }$. The length of $\sum_{e\in E}e$ is a constant $\sqrt{|E|}$, and it can take any direction as $E$ varies. 
Therefore, when $E$ changes to align the directions of $\sum_{e\in E} (e\cdot \sum_{u\in g(B,\theta)}u)e$ and 
$\sum_{e\in E}e$, the term 
$(\sum_{e\in E}e)\cdot(\sum_{e\in E} (e\cdot \sum_{u\in g(B,\theta)}u)e)$ achieves its maximum value $\sqrt{\sum_{e\in E} (e\cdot \sum_{u\in g(B,\theta)}u)^2 }\times \sqrt{|E|} = \sqrt{|E|\sum_{e\in E} (e\cdot \sum_{u\in g(B,\theta)}u)^2 }$, which remains constant as 
$E$ changes in $\gE(g(S,\theta))$. Thus, $\forall E \in \gE(g(S,\theta))$, we have 
\begin{equation}
    \sqrt{|E|\sum_{e\in E} (e\cdot \sum_{u\in g(B,\theta)}u)^2 } = r(S,B,\theta).
\end{equation}

\end{proof}

\subsection{Proof of \cref{prop:r'}}
\begin{definition}[Proposition 2.3 in \citet{DBLP:journals/ftml/Bach13}]\label{definition:submodular}
The set-function $F$ is submodular if and only if for all $S\subset B$, and $j,k \in B \setminus S$, we have
$$
F(S\cup\{k\})-F(S) \geq F(S \cup \{j,k\})-F(S \cup \{j\}).
$$
And the function is called normalized if $F(\emptyset)= 0$ and monotone if and only if $F(S’) \leq F(S)$, $\forall S'\subset S$.
\end{definition}

\begin{proof}[Proof of \cref{prop:r'}]
\ \\
Given $r'(S,B,\theta) = \sqrt{\sum_{e\in E} (e\cdot \sum_{u\in g(B,\theta)}u)^2 }$, $\forall E \in \gE(g(S,\theta))$. Note that $r'$ remains constant for arbitrary choice of the basis $E$ spanned the same subspace. 
For all $S\subset B$, and $j,k \in B \setminus S$, we discuss this problem case by case. 

(1) $g(j,\theta)$ is in the subspace spanned by $g(S,\theta)$, \IE, $g(j,\theta) - \sum_{e\in E}(e\cdot g(j,\theta))e = 0$. we have 
\begin{equation}
    \begin{split}
        &r'(S,B,\theta) = r'(S\cup\{j\},B,\theta)\\
        &r'(S\cup\{k\},B,\theta) = r'(S\cup\{j,k\},B,\theta)
    \end{split}
\end{equation}
Thus, we have 
\begin{equation}
    r'(S\cup\{k\},B,\theta) - r'(S,B,\theta) = r'(S\cup\{j,k\},B,\theta) - r'(S\cup\{j\},B,\theta).
\end{equation}

(2) $g(j,\theta)$ is not in the subspace spanned by $g(S,\theta)$, \IE, $g(j,\theta) - \sum_{e\in E}(e\cdot g(j,\theta))e \neq 0$. Define $e_j$ as 
\begin{equation}
    e_j = \frac{g(j,\theta) - \sum_{e\in E}(e\cdot g(j,\theta))e}{\|g(j,\theta) - \sum_{e\in E}(e\cdot g(j,\theta))e\|},
\end{equation}
we have $E\cup\{e_j\}\in \gE(g(S\cup\{j\},\theta))$.

\hspace{2em}(2a) $g(k,\theta)$ is in the subspace spanned by $g(S\cup\{j\},\theta)$. We have
\begin{equation}
    \begin{split}
        &r'(S\cup\{j,k\},B,\theta) - r'(S\cup\{j\},B,\theta) = 0\\
        &r'(S\cup\{k\},B,\theta) - r'(S,B,\theta) \geq 0 = r'(S\cup\{j,k\},B,\theta) - r'(S\cup\{j\},B,\theta)
    \end{split}
\end{equation}

\hspace{2em}(2b) $g(k,\theta)$ is not in the subspace spanned by $g(S\cup\{j\},\theta)$. Define $e_k$ as
\begin{equation}
    e_k = \frac{g(k,\theta) - \sum_{e\in E}(e\cdot g(k,\theta))e - (e_j\cdot g(k,\theta))e_j}{\|g(k,\theta) - \sum_{e\in E}(e\cdot g(k,\theta))e - (e_j\cdot g(k,\theta))e_j\|},
\end{equation}

\hspace{2em}we have
\begin{equation}
    \begin{split}
        &E\cup\{e_k\}\in \gE(g(S\cup\{k\},\theta)) \\
        &E\cup\{e_j,e_k\}\in \gE(g(S\cup\{j,k\},\theta))
    \end{split}
\end{equation}

\hspace{2em}Then
\begin{equation}
    \begin{split}
        & r'(S,B,\theta) = \sqrt{\sum_{e\in E} (e\cdot \sum_{u\in g(B,\theta)}u)^2} \\
        & r'(S\cup\{k\},B,\theta) = \sqrt{\sum_{e\in E} (e\cdot \sum_{u\in g(B,\theta)}u)^2 + (e_k\cdot \sum_{u\in g(B,\theta)}u)^2} \\
        & r'(S\cup\{j\},B,\theta) = \sqrt{\sum_{e\in E} (e\cdot \sum_{u\in g(B,\theta)}u)^2 + (e_j\cdot \sum_{u\in g(B,\theta)}u)^2} \\
        & r'(S\cup\{j,k\},B,\theta) = \sqrt{\sum_{e\in E} (e\cdot \sum_{u\in g(B,\theta)}u)^2 + (e_j\cdot \sum_{u\in g(B,\theta)}u)^2 + (e_k\cdot \sum_{u\in g(B,\theta)}u)^2} \\
    \end{split}
\end{equation}

\hspace{2em}
Note that the function $\hat{\beta}(x)=\sqrt{a+x}-\sqrt{x}$ with $a>0$ is a decreasing function w.r.t. $x$. Let $a=(e_k\cdot \sum_{u\in g(B,\theta)}u)^2$, $x_1 = \sum_{e\in E} (e\cdot \sum_{u\in g(B,\theta)}u)^2$ and $x_2 = \sum_{e\in E} (e\cdot \sum_{u\in g(B,\theta)}u)^2 + (e_j\cdot \sum_{u\in g(B,\theta)}u)^2 $,  we have $x_1<x_2$ and  $\hat{\beta}(x_1)>\hat{\beta}(x_2)$. It follows that
\begin{equation}
    r'(S\cup\{k\},B,\theta) - r'(S,B,\theta) > r'(S\cup\{j,k\},B,\theta) - r'(S\cup\{j\},B,\theta)
\end{equation}

In summary, for all $S\subset B$, and $j,k \in B \setminus S$, we have
\begin{equation}
    r'(S\cup\{k\},B,\theta) - r'(S,B,\theta) \geq r'(S\cup\{j,k\},B,\theta) - r'(S\cup\{j\},B,\theta)
\end{equation}

And it's obvious that $r'(S,B,\theta)$ is normalized and monotone.
    
\end{proof}

\clearpage
\subsection{Proof of \cref{prop:r guarantee}}\label{appendix:Proof of prop:r guarantee}

\begin{lemma}[\citet{DBLP:journals/mp/NemhauserWF78}]\label{lemma: 1-e-1}
Greedy maximization of a monotone, submodular function returns a set with value within a factor of $1-e^{-1}$ from the optimum set with the same size.
\end{lemma}

\begin{lemma}\label{lemma:r' ratio}
    \cref{alg:greedy} returns a $1-e^{-1}$ approximation for $\arg\max_{S\subset B} r'(S,B,\theta), \st |S|\leq N_S$. That is, denote $S^{**}$ as the optimum subset for maximizing $r'(S,B,\theta)$, and $S'$ is the output of \cref{alg:greedy}, we have
$$r'(S',B,\theta) \geq (1-e^{-1}) r'(S^{**},B,\theta).$$
\end{lemma}

\begin{proof}[Proof of \cref{lemma:r' ratio}]\ \\
    Note that \cref{alg:greedy} is also a greedy maximization of $r'(S,B,\theta)$, given \cref{prop:r'} and \cref{lemma: 1-e-1}, the conclusion can be drawn immediately.\\
\end{proof}

\begin{proof}[Proof of \cref{prop:r guarantee}]\ \\
 Denote $S^{*}$ as the optimum subset for maximizing $r(S,B,\theta)$, $S^{**}$ as the optimum subset for maximizing $r'(S,B,\theta)$, and $S'$ as the output of \cref{alg:greedy}. Let $E^*,E^{**},E'$ be orthonormal bases corresponding to $S^{*}$, $S^{**}$, $S'$, respectively. We can get
 \begin{equation}
     |E^*| = |E^{**}| = |E'| = \min(\mathrm{rank}(B),N_S).
 \end{equation}

Otherwise, If $|E| < \min(\mathrm{rank}(B),N_S)$, there must be at least one selected $d_i$ that does not contribute to $E$, and an unselected $d_j$ that can contribute a new element to $E$. Therefore, replacing $d_i$ with $d_j$ further increases the objective value, both from a global and a sequential greedy perspective. 

Based on \cref{lemma:r' ratio}, we have
\begin{equation}
    \begin{split}
        r(S',B.\theta) &= \sqrt{\min(\mathrm{rank}(B),N_S)}r'(S',B.\theta)\\
        &\geq (1-e^{-1}) \sqrt{\min(\mathrm{rank}(B),N_S)}r'(S^{**},B,\theta)\\
        &\geq (1-e^{-1}) \sqrt{\min(\mathrm{rank}(B),N_S)}r'(S^{*},B,\theta)\quad (\text{Definition of }S^{**})\\
        &= (1-e^{-1})r(S^{*},B,\theta)
    \end{split}
\end{equation}

\end{proof}

\subsection{Plain-text Description of \cref{alg:greedy,alg:DivBS}}
\label{appendix:plain-text for algs}
\cref{alg:greedy}:
\begin{enumerate}
	\item Initialize the selected subset $S$ to an empty set and the corresponding orthonormal basis $E$ to an empty set. Denote the sum of all features of elements of batch $B$ as $\mathrm{Sum}$: $S \leftarrow \emptyset$, $E \leftarrow \emptyset$, $\mathrm{Sum} \leftarrow \sum_{u\in g(B,\theta)}u$.
	\item Add a sample to $S$ that maximizes the current $r(S,B,\theta) = \sqrt{|E|\sum_{e\in E} (e\cdot \sum_{u\in g(B,\theta)}u)^2 }$, where $E \in \mathcal{E}(g(S,\theta))$.
	\begin{enumerate}
	\item Compute the contribution of each candidate sample when individually added to $S$ to the orthonormal basis $E$: $E_{\text{Cand}} \leftarrow {\frac{g(d,\theta) - \sum_{e \in E} (e \cdot g(d,\theta)) e}{|g(d,\theta) - \sum_{e \in E} (e \cdot g(d,\theta)) e|} \text{ for } d \in B}$
	\item Identify the sample that maximizes $r$: $\text{idx} \leftarrow \arg\max_{i} |e_i\cdot \mathrm{Sum}|, e_i \in E_{\text{Cand}}$
	\item Update $E$, $S$, $B$: $S \leftarrow S \cup\{d_{\text{idx}}\}$, $E \leftarrow E \cup\{e_{\text{idx}}\}$, $B \leftarrow B \setminus \{d_{\text{idx}}\}$
	\end{enumerate}
	\item Repeat Step 2 until $|S| = N_S$.
\end{enumerate}

\cref{alg:DivBS}:
\begin{enumerate}
	\item Initialize the selected subset $S$ to an empty set and the corresponding orthonormal basis $E$ to an empty set. Initialize $\mathrm{Sum}$ as the sum of all features of elements of batch $B$: $\leftarrow \emptyset$, $E \leftarrow \emptyset$, $\mathrm{Sum} \leftarrow \sum_{u\in g(B,\theta)}u$.
	\item Add a sample to $S$ that approximately maximizes the current $r$, referring to \cref{appendix:Approximation from alg:greedy to alg:DivBS} for the approximation.
	\begin{enumerate}
	\item  Select the sample $d$ to be added at the current step: $d \leftarrow \arg\max_{d \in B} |g(d,\theta) \cdot \mathrm{Sum}|$
	\item Compute the contribution of sample $d$ to the current orthonormal basis: $e \leftarrow \frac{g(d,\theta) - \sum_{e \in E} (e \cdot g(d,\theta)) e}{|g(d,\theta) - \sum_{e \in E} (e \cdot g(d,\theta)) e|}$
	\item Update $E$, $S$, $B$: $S \leftarrow S \cup\{d\}$, $E \leftarrow E \cup\{e\}$, $B \leftarrow B \setminus\{d\}$
	\end{enumerate}
\end{enumerate}

\subsection{Approximation from \cref{alg:greedy} to \cref{alg:DivBS}}\label{appendix:Approximation from alg:greedy to alg:DivBS}

In \cref{alg:greedy}, we select $d_{greedy}$ {as follows} :
\begin{equation}
    d_{greedy} = \arg\max_{d\in B} |\frac{g(d,\theta) - \sum_{e \in E} (e \cdot g(d,\theta)) e}{\|g(d,\theta) - \sum_{e \in E} (e \cdot g(d,\theta)) e\|}\cdot \sum_{u\in g(B,\theta)}u|
\end{equation}
where $E$ represents an orthogonal basis corresponding to the already selected samples (line 7,8 in \cref{alg:greedy}). 
If we disregard the normalization term $\|g(d,\theta) - \sum_{e \in E} (e \cdot g(d,\theta)) e\|$, then
\begin{equation}
   \begin{split}
&|{g(d,\theta) - \sum_{e \in E} (e \cdot g(d,\theta)) e}\cdot \sum_{u\in g(B,\theta)}u| \\&= |(g(d,\theta) - \sum_{e \in E} (e \cdot g(d,\theta)) e)\cdot(\sum_{u\in g(B,\theta)}u - \sum_{e\in E}(e\cdot\sum_{u\in g(B,\theta)}u)e +\sum_{e\in E}(e\cdot\sum_{u\in g(B,\theta)}u)e)|\\
&= |(g(d,\theta) - \sum_{e \in E} (e \cdot g(d,\theta)) e)\cdot(\sum_{u\in g(B,\theta)}u - \sum_{e\in E}(e\cdot\sum_{u\in g(B,\theta)}u)e |\\
&= |g(d,\theta)\cdot(\sum_{u\in g(B,\theta)}u - \sum_{e\in E}(e\cdot\sum_{u\in g(B,\theta)}u)e |
   \end{split} 
\end{equation}

The third and fourth lines follow from the fact that for any $e_i\in E$, we have $(g(d,\theta) - \sum_{e \in E} (e \cdot g(d,\theta)) e)\cdot e_i = 0$, and $(\sum_{u\in g(B,\theta)}u - \sum_{e\in E}(e\cdot\sum_{u\in g(B,\theta)}u)e)\cdot e_i=0$. Note that $\arg\max_{d\in B} |g(d,\theta)\cdot(\sum_{u\in g(B,\theta)}u - \sum_{e\in E}(e\cdot\sum_{u\in g(B,\theta)}u)e) |$ corresponds to line 5 of \cref{alg:DivBS} given $\mathrm{Sum} = \sum_{u\in g(B,\theta)}u - \sum_{e\in E}(e\cdot\sum_{u\in g(B,\theta)}u)e$ (line 6,8,9 of \cref{alg:DivBS}), 
{we can employ \cref{alg:DivBS} to approximate \cref{alg:greedy}}.

\section{Supplement for Experiments}
\subsection{Toy Example (\cref{fig: toy})}\label{appendix:toy}
In \cref{fig: toy}, we visualize a toy motivating example involving a four-class classification task among red, blue, green and yellow points. Specifically, we sample 1000 red points, 300 blue points, 150 green points, and 20 yellow points from following normal distributions: $N([0,0],[1,1])$, $N([5,0],[1,1])$, $N([0,5],[1,1])$, and $N([5,5],[1,1])$, respectively. We use a two-layer MLP with 100 hidden neurons as the model for the toy study. We construct a batch $B$ using all available training data. The toy models are trained using Adam with learning rate 0.001 for 100 epochs. The budget ratio is set to 10\%.

\subsection{T-SNE Visualization}
In \cref{fig:tsne}, we visualize subsets selected by different methods from the same batch of data {on CIFAR-10. The batch size in all the experiments is set to 320. $10\%$-budget, i.e., 32 samples are selected. The t-SNE~\citep{JMLR:v9:vandermaaten08a} visualization of the  last layer features is shown in \cref{fig:tsne}.} The red points represent the selected samples, and the gray points represent the full data. We have circled highly redundant samples. It is evident that baseline methods tend to select redundant samples, wasting data capacity, while our \methodspace effectively avoids such issues.

\begin{figure}[t]
    \centering
    \includegraphics[width=0.8\textwidth]{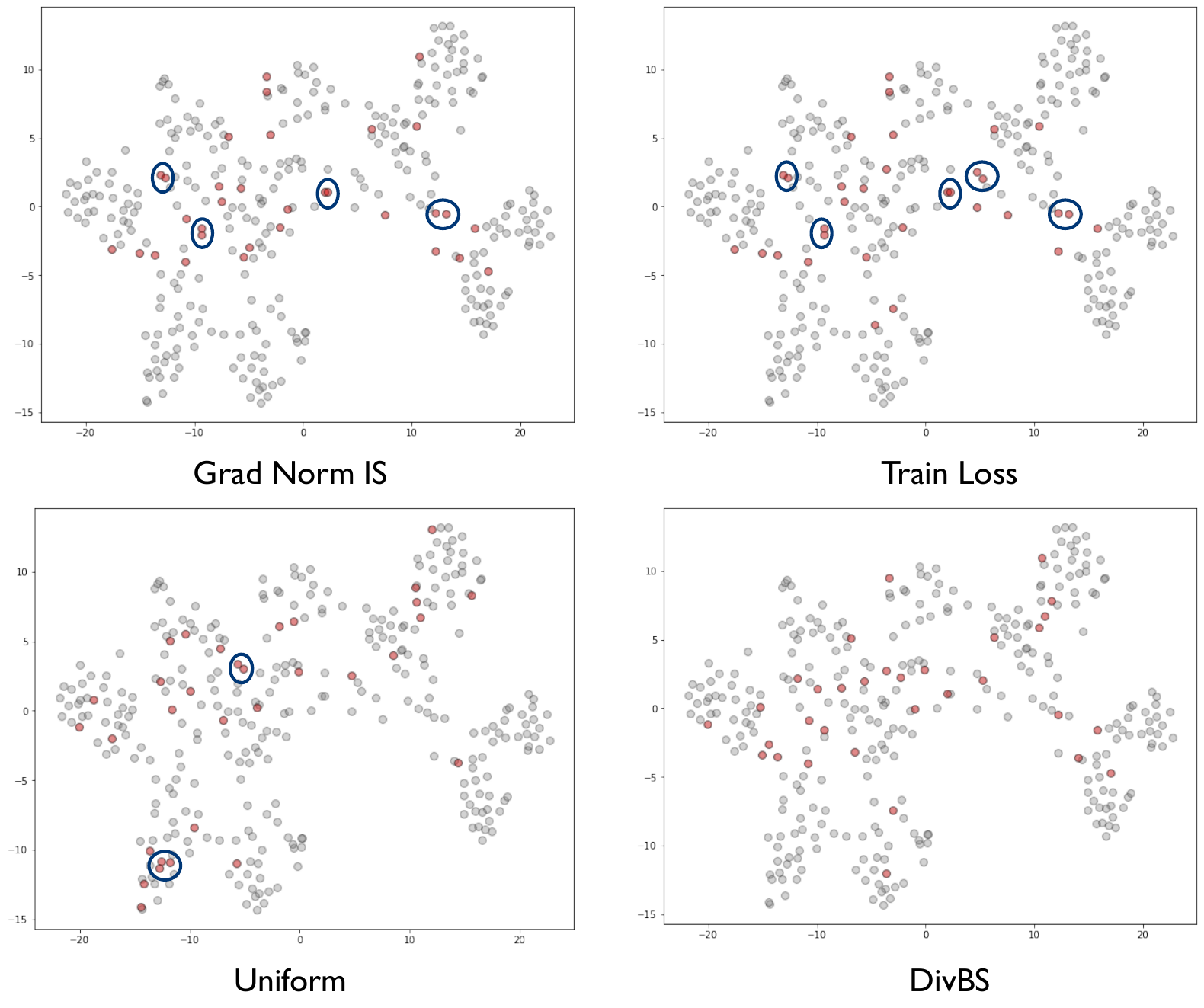}
    \caption{T-SNE visualization of data selected by different methods on CIFAR-10 with 10\% budget. Circles highlight redundant samples.}
    \label{fig:tsne}
\end{figure}

\subsection{Comparison with InfoBatch~\citep{qin2024infobatch}}
In this section, we compare the method InfoBatch~\citep{qin2024infobatch}. As it doesn't originally operate with a fixed budget (and the average budget exceeds our setup), we adapt it by using a percentile threshold based on the guidance\footnote{\urlinfobatch} from the official repository to fix the budget ratio. The results are presented in \cref{tab:com_w_infobath}. Note that, the results of InfoBatch are significantly lower than those of Uniform Sampling. This is due to the rescaling operation in InfoBatch, which can lead to unstable training, especially under the small budget. For example, with a budget of 10\%, if we set the percentile threshold for infobatch to 95\%, then infobatch needs to assign a weight of $\frac{0.95}{0.1-0.05}=19$ times to the selected low-loss samples, which clearly destabilizes the training process. In contrast, the default threshold for InfoBatch is the mean, and it only clips 50\% of samples below the mean, assigning a weight of $2$ to low-loss samples. This may suggest that, rather than weighting losses of samples, simply selecting a subset of samples may be safer for accelerating model training.

\subsection{Comparison with K-means++ Initialization Method}

Given that K-means++ initialization method is also build to sample diverse centers, we conduct comparison between DivBS and k-means++ initialization method in \cref{tab:kmeanspp}. K-means++ initialization has shown comparable results to Uniform sampling, achieving some improvement at 10\% and 20\% budget ratios. This can be attributed to K-means++ initialization placing some emphasis on diversity, which may become more important as the budget decreases. However, DivBS still outperforms K-means++ initialization significantly, possibly due to the following reasons: 1) K-means++ initialization, as a heuristic algorithm, probabilistically samples points based on their distances from previously selected points, without a stable performance guarantee; 2) K-means++ initialization only considers distances between subset elements without considering the representativeness of the selected subset for the entire batch. In contrast, our proposed objective simultaneously considers both the diversity within the subset and the representativeness of the subset.

\begin{table}[t]
    \centering
    \caption{Accuracies on CIFAR-100 for different budget ratios.}
    \label{tab:com_w_infobath}
    \begin{tabular}{cccc}
        \toprule
        {Budget ratio} & {10\%} & {20\%} & {30\%} \\
        \midrule
        Uniform & 70.61\% & 74.18\% & 75.98\% \\
        InfoBatch & 43.56\% & 68.22\% & 74.31\% \\
        \textbf{DivBS} & \textbf{73.11\%} & \textbf{76.10\%} & \textbf{77.21\%} \\
        \bottomrule
    \end{tabular}    
\end{table}

\begin{table}[t]
    \centering
    \caption{Final accuracies of DivBS and K-means++ initialization method on CIFAR-100 for different budget ratios. Full Data Training achieves 77.28\% accuracy.}\label{tab:kmeanspp}
    \begin{tabular}{cccc}
        \toprule
        {Budget ratio} & {10\%} & {20\%} & {30\%} \\
        \midrule
        Uniform & 70.61\% & 74.18\% & 75.98\% \\
        k-means++ initialization & 71.14\% & 74.35\% & 75.76\% \\
        \textbf{DivBS} & \textbf{73.11\%} & \textbf{76.10\%} & \textbf{77.21\%} \\
        \bottomrule
    \end{tabular}
\end{table}

\subsection{Error Bars}
We provide error bars of \cref{tab: classification-main} in \cref{tab: error bars}.

\begin{table*}[t]
\centering
\caption{Final accuracies ($\uparrow$, mean $\pm$ std) of \methodspace and  various baseline methods on CIFAR-10, CIFAR-100 and Tiny ImageNet with different budget ratio 10\%, 20\%, 30\%. Bold indicates the best results. Experiments show that \methodspace consistently outperforms all baselines.}
\label{tab: error bars}
 \vskip 0.02in
\resizebox{\textwidth}{!}{
\begin{tabular}{c|ccc|ccc|ccc}
\toprule
Method                 & \multicolumn{3}{c|}{CIFAR-10} & \multicolumn{3}{c|}{CIFAR-100} & \multicolumn{3}{c}{Tiny ImageNet} \\
\midrule\midrule
Full Data Training & \multicolumn{3}{c|}{95.50\%}         & \multicolumn{3}{c|}{77.28\%}          & \multicolumn{3}{c}{56.76\%}            \\
Budget ratio                & 10\%          & 20\%     & 30\%    & 10\%          & 20\%          & 30\%            & 10\%           & 20\%            & 30\%           \\
\midrule
        \multirow{2}{*}{Uniform} & 92.06\% & 93.76\% & 94.61\% & 70.61\% & 74.18\% & 75.98\% & 48.36\% & 51.71\% & 53.76\% \\
        & \textcolor{gray}{$\pm$ 0.19\%} & \textcolor{gray}{$\pm$ 0.14\%} & \textcolor{gray}{$\pm$ 0.19\%} & \textcolor{gray}{$\pm$ 0.34\%} & \textcolor{gray}{$\pm$ 0.37\%} & \textcolor{gray}{$\pm$ 0.31\%} & \textcolor{gray}{$\pm$ 0.23\%} & \textcolor{gray}{$\pm$ 0.28\%} & \textcolor{gray}{$\pm$ 0.32\%} \\
        \multirow{2}{*}{Train Loss} & 92.73\% & 93.87\% & 94.54\% & 65.12\% & 69.34\% & 72.62\% & 37.12\% & 45.23\% & 47.72\% \\
        & \textcolor{gray}{$\pm$ 0.22\%} & \textcolor{gray}{$\pm$ 0.16\%} & \textcolor{gray}{$\pm$ 0.21\%} & \textcolor{gray}{$\pm$ 0.36\%} & \textcolor{gray}{$\pm$ 0.31\%} & \textcolor{gray}{$\pm$ 0.37\%} & \textcolor{gray}{$\pm$ 0.25\%} & \textcolor{gray}{$\pm$ 0.30\%} & \textcolor{gray}{$\pm$ 0.25\%} \\
        \multirow{2}{*}{Grad Norm} & 65.23\% & 76.23\% & 82.34\% & 64.72\% & 69.23\% & 72.34\% & 37.24\% & 44.34\% & 48.24\% \\
        & \textcolor{gray}{$\pm$ 0.17\%} & \textcolor{gray}{$\pm$ 0.20\%} & \textcolor{gray}{$\pm$ 0.24\%} & \textcolor{gray}{$\pm$ 0.29\%} & \textcolor{gray}{$\pm$ 0.34\%} & \textcolor{gray}{$\pm$ 0.28\%} & \textcolor{gray}{$\pm$ 0.21\%} & \textcolor{gray}{$\pm$ 0.26\%} & \textcolor{gray}{$\pm$ 0.31\%} \\
        \multirow{2}{*}{Grad Norm IS} & 92.51\% & 93.78\% & 94.41\% & 69.34\% & 72.71\% & 73.21\% & 42.79\% & 47.34\% & 50.23\% \\
        & \textcolor{gray}{$\pm$ 0.20\%} & \textcolor{gray}{$\pm$ 0.25\%} & \textcolor{gray}{$\pm$ 0.16\%} & \textcolor{gray}{$\pm$ 0.35\%} & \textcolor{gray}{$\pm$ 0.30\%} & \textcolor{gray}{$\pm$ 0.26\%} & \textcolor{gray}{$\pm$ 0.23\%} & \textcolor{gray}{$\pm$ 0.28\%} & \textcolor{gray}{$\pm$ 0.13\%} \\
        \multirow{2}{*}{SVP} & 57.38\% & 73.87\% & 82.34\% & 31.23\% & 43.35\% & 50.73\% & 19.34\% & 28.97\% & 34.24\% \\
        & \textcolor{gray}{$\pm$ 0.15\%} & \textcolor{gray}{$\pm$ 0.08\%} & \textcolor{gray}{$\pm$ 0.22\%} & \textcolor{gray}{$\pm$ 0.27\%} & \textcolor{gray}{$\pm$ 0.32\%} & \textcolor{gray}{$\pm$ 0.26\%} & \textcolor{gray}{$\pm$ 0.17\%} & \textcolor{gray}{$\pm$ 0.22\%} & \textcolor{gray}{$\pm$ 0.27\%} \\
        \multirow{2}{*}{Moderate-BS} & 92.32\% & 93.57\% & 94.36\% & 70.21\% & 74.35\% & 75.34\% & 48.92\% & 51.36\% & 54.23\% \\
        & \textcolor{gray}{$\pm$ 0.18\%} & \textcolor{gray}{$\pm$ 0.23\%} & \textcolor{gray}{$\pm$ 0.18\%} & \textcolor{gray}{$\pm$ 0.23\%} & \textcolor{gray}{$\pm$ 0.18\%} & \textcolor{gray}{$\pm$ 0.24\%} & \textcolor{gray}{$\pm$ 0.23\%} & \textcolor{gray}{$\pm$ 0.20\%} & \textcolor{gray}{$\pm$ 0.32\%} \\
        \multirow{2}{*}{CCS-BS} & 92.61\% & 93.88\% & 94.81\% & 71.11\% & 74.42\% & 76.21\% & 49.18\% & 52.43\% & 54.17\% \\
        & \textcolor{gray}{$\pm$ 0.21\%} & \textcolor{gray}{$\pm$ 0.16\%} & \textcolor{gray}{$\pm$ 0.11\%} & \textcolor{gray}{$\pm$ 0.36\%} & \textcolor{gray}{$\pm$ 0.31\%} & \textcolor{gray}{$\pm$ 0.37\%} & \textcolor{gray}{$\pm$ 0.26\%} & \textcolor{gray}{$\pm$ 0.21\%} & \textcolor{gray}{$\pm$ 0.36\%} \\
        \multirow{2}{*}{\textbf{DivBS}} & \textbf{94.65\%} & \textbf{94.83\%} & \textbf{95.07\%} & \textbf{73.11\%} & \textbf{76.10\%} & \textbf{77.21\%} & \textbf{50.84\%} & \textbf{55.03\%} & \textbf{55.94\%} \\
        & \textcolor{gray}{$\pm$ 0.27\%} & \textcolor{gray}{$\pm$ 0.22\%} & \textcolor{gray}{$\pm$ 0.27\%} & \textcolor{gray}{$\pm$ 0.12\%} & \textcolor{gray}{$\pm$ 0.37\%} & \textcolor{gray}{$\pm$ 0.23\%} & \textcolor{gray}{$\pm$ 0.28\%} & \textcolor{gray}{$\pm$ 0.23\%} & \textcolor{gray}{$\pm$ 0.15\%} \\
        \bottomrule    
\end{tabular}
}
\end{table*}

\section{Other Sampling Methods Involving Diversity or Submodularity}
Curriculum learning involving redundency and diversity: self-paced learning with diversity (SPLD)~\citep{DBLP:conf/nips/JiangMYLSH14} is the first work to introduce diversity into curriculum learning, formalizing preferences for simple and diverse samples as a universal regularization term. MCL~\citep{DBLP:conf/iclr/ZhouB18} proposes that early training should focus on a small set of diverse samples, while later stages should prioritize training on larger, more challenging, and more homogeneous samples. Similar to MCL, DoCL~\citep{DBLP:conf/aistats/ZhouWB21} promotes diversity through regularization using a submodular function.

Submodular coreset selection: Craig~\citep{DBLP:conf/icml/MirzasoleimanBL20} attempt to find an coreset that approximates the gradients of the full dataset. They achieve this by transforming the gradient matching problem into the maximization of a monotone submodular. GLISTER~\citep{DBLP:conf/aaai/KillamsettySRI21} formulates coreset selection as a mixed discrete-continuous bi-level optimization problem. It aims to select a subset of the training data that maximizes the log-likelihood on a held-out validation set. Additionally, GLISTER establishes connections to submodularity.

Submodular active learning: \citet{DBLP:conf/icml/WeiIB15} discusses the connection between likelihood functions and submodularity. It demonstrates that under a cardinality constraint, maximizing the likelihood function is equivalent to maximizing submodular functions for Naive Bayes or Nearest Neighbor classifiers. This naturally provides a powerful tool for coreset selection. By introducing submodularity into Naive Bayes and Nearest Neighbor classifiers, they propose a novel framework for active learning called Filtered Active Submodular Selection (Fass). CAL~\citep{das2023accelerating} integrates continual learning techniques into active learning to alleviate the high training costs associated with active learning. Similarly, it employs submodular functions to regularize the sampling points.

Diversity-aware active learning: some sampling methods also focus on diversity of the chosen elements in the active learning area~\citep{ren2021survey}, where a model proactively chooses and queries the most informative data points for annotation, aiming to enhance its performance with minimal labeled examples. For example, \citet{sener2018active} theoretically formalize the data selection process as a k-Center problem and introduce the CoreSet algorithm, while \citet{agarwal2020contextual} substitute the Euclidean distance with context-aware KL-divergence. Determinantal Point Process (DPP)~\citep{kulesza2012determinantal,tremblay2019determinantal} is also an effective sampling method for preventing redundancy. 

Discussion: the majority of sampling methods involving diversity come with a high computational cost, requiring at least $O(N^2)$ or $O(N^3)$ to calculate set properties (such as pairwise distances or determinants) and $O(N^2)$ or $O(N^3)$ to perform the sampling process, where $N$ is the number of all candidate elements. As a result, they are suitable only for small-scale offline sampling and are not applicable for large-scale online selection. In contrast, our method has been demonstrated to be sufficiently lightweight, enabling its application in accelerating training within the online batch selection paradigm.

\section{Limitation and Future Work}
Online batch selection methods require an additional forward pass for selecting a subset in each batch, which somewhat limits the upper bound of acceleration, especially when the budget is small. Similar to previous research of online batch selection, our scope is also limited to the effects of different selection strategies. Exploring the integration of techniques discussed in \cref{sec:related work} such as hardware acceleration for forward pass, parallelization, or leveraging historical training information to avoid additional data loading and forward pass for maximum acceleration is a promising and noteworthy avenue.